\ifdefined \TOG
\documentclass{acmtog}
\acmVolume{V}
\acmNumber{N}
\acmYear{20YY}
\acmMonth{M}
\acmArticleNum{A}  
\usepackage{amsmath, amssymb}
\usepackage[percents]{overpic}
\newcommand{\eqmark}{&&}
\newcommand{\eqcr}{\cr}
\else
\documentclass[12pt]{article}
\usepackage[utf8x]{inputenc}
\usepackage{overpic}
\usepackage{graphicx}
\usepackage{amsmath, amssymb, amsthm, eqnarray}
\usepackage{dsfont}
\usepackage{setspace}
\usepackage{csquotes}
\usepackage[pdftex]{hyperref}
\usepackage[english]{babel}
\usepackage{hebfont}
\usepackage[toc,page]{appendix}
\usepackage{float}
\usepackage{enumitem}
\usepackage{bookmark}

\usepackage{anysize}

\newtheorem{theorem}{Theorem}

\newtheorem{proposition}{Proposition}
\newtheorem{lemma}{Lemma}

\newcommand{\eqmark}{}
\newcommand{\eqcr}{}

\fi
\usepackage{color}
\usepackage{algorithm}
\usepackage{algpseudocode}
\usepackage{caption}
\usepackage{mathtools}

\usepackage{xspace}
\makeatletter

\DeclareRobustCommand\onedot{\futurelet\@let@token\@onedot}
\def\@onedot{\ifx\@let@token.\else.\null\fi\xspace}
\def\ie{{i.e}\onedot}
\def\eg{{e.g}\onedot}
\def\etal{{et al}\onedot}
\DeclareMathOperator*{\argmin}{argmin}

\DeclareMathOperator{\eqs}{\:\:\,}
\newcommand{\norm}[1]{\left\lVert#1\right\rVert}
\newcommand{\abs}[1]{\lvert#1\rvert}

\begin{document}

\title{
 On Nonrigid Shape Similarity and Correspondence  
 }

\ifdefined \TOG
\markboth{A. Shtern and R. Kimmel}{
   On Non-Rigid Shape Similarity and Correspondence  
  }
\author{Alon Shtern {\upshape and} Ron Kimmel
\affil{Technion - Israel Institute of Technology}}



\else
\setcounter{secnumdepth}{5}
\author{Alon Shtern {\upshape and} Ron Kimmel}
\fi

\maketitle

\begin{figure}[th]
	\begin{center}
	\begin{overpic}[width=1.0\columnwidth]{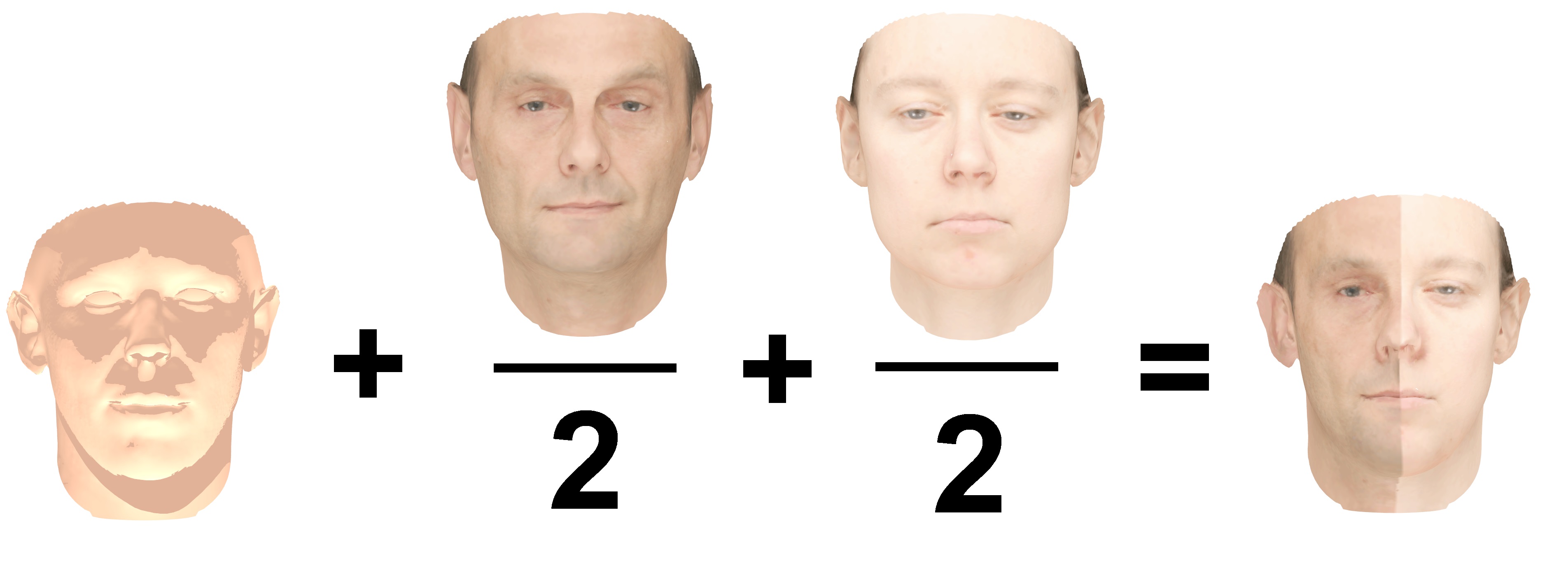}
	\end{overpic}
	\end{center}
	\caption{\small Use of correspondence for symmetry detection and texture transfer. The two intrinsically symmetric halves of a human face were found by mapping the shape (left) to itself. Textures from two faces (middle) were transferred to each half (right).}
	\label{fig:face_math}
\end{figure}


\begin{abstract}
\label{sec:abstract}
An important operation in geometry processing is finding the correspondences 
 between pairs of shapes.
The Gromov-Hausdorff distance, a measure of dissimilarity between metric spaces, 
 has been found to  be highly useful for nonrigid shape comparison.
Here, we explore the applicability of related shape similarity measures to the problem of shape correspondence,
adopting spectral type distances.
 We propose to evaluate the {\em spectral kernel distance}, the {\em spectral embedding distance} and the novel {\em spectral quasi-conformal distance}, comparing the manifolds from different viewpoints. 
By matching the shapes in the spectral domain, important attributes of surface structure are being aligned.
For the purpose of testing our ideas, we introduce a fully automatic framework for finding intrinsic
 correspondence between two shapes.
The proposed method achieves state-of-the-art results on the Princeton isometric shape 
 matching protocol applied, as usual, to the TOSCA and SCAPE benchmarks.
\end{abstract}


\section{Introduction}
\label{sec:introduction}

Correspondence detection between pairs of shapes lies at the heart of many operations in the field of
  geometry processing.
The problem of acquiring correspondence between rigid shapes has been widely addressed in the literature. 
As for non-rigid shapes, this problem remains difficult even when the space of deformations 
 is narrowed to nearly isometric surfaces, which approximately preserve the geodesic distances between 
 corresponding points on each shape.

A common approach for shape matching is to define a measure of dissimilarity between shapes modeled as 2-manifolds.
The well-established Gromov-Hausdorff distance 
 measures the maximum geodesic discrepancy between pairs of corresponding points of the two given shapes \cite{memoli2005theoretical,bronstein2006generalized}.
The point-wise map can be inferred to as a byproduct of the evaluation of the Gromov-Hausdorff distance.
This approach was embraced by the {\em Generalized Multi-Dimensional Scaling} (GMDS) framework 
 \cite{bronstein2006generalized}, and by its recently proposed
  {\em Spectral} GMDS implicit  alternative,  in which the distances matching problem is treated in the 
 dual intrinsic spectral domains of the given shapes \cite{aflalo2013sgmds}.
Within the Gromov-Hausdorff framework, Bronstein \etal \cite{bronstein2010gromov} suggested replacing the 
 geodesic distance by the diffusion distance \cite{coifman2005geometric}, exploiting the apparent stability of diffusion distances to local 
 changes in the topology of the shape.
Despite its generality and theoretical beauty, it has been a challenge to apply
 the Gromov-Hausdorff framework in a strait-forward manner to shape matching,
 mainly due to its intrinsically combinatorial nature.

B{\'e}rard \etal \cite{berard1994embedding} exhibited the spectral embedding of Riemannian manifolds 
 by their heat kernel. 
They embedded the manifolds into a compatible common Euclidean space
  and proved that the Hausdorff distance in that space is a metric between isometry classes of Riemannian manifolds, which means, in particular, that two manifolds are at zero distance
  if and only if they are isometric. 

Lipman \etal introduced the {\em conformal Wasserstein distance} between surfaces with restricted disk or sphere 
 topology \cite{lipman2011conformal,lipman2009mobius}, utilizing the M{\"o}bius transform.
The definition of the conformal Wasserstein distance is based on the well-known class of transportation
 distances \cite{mumford1991mathematical,rubner2000earth} 
 (a.k.a. Wasserstein-Kantorovich-Rubinstein or Earth Mover’s Distance).
Several authors have suggested to generalize the class of M{\"o}bius transforms, \eg using {\em least squares conformal maps} \cite{wang2007conformal}, bounding the {\em conformal factor} \cite{aflalo2013conformal}, or controlling the {\em conformal distortion} \cite{lipman2012bounded,zeng2011registration,ahlfors1966lectures}.
Zhu \etal suggested to apply post-process flattening, and impose area preserving constraints \cite{zhu2003area}.
However, it is known that in the general case, a mapping cannot be both angle-preserving and area-preserving.

Kasue and Kumura \cite{kasue1994spectral} extended the Gromov-Hausdorff distance framework to the class
 of spectral methods.
The {\em spectral kernel distance} was constructed by replacing the metric defined on the manifolds with the heat kernel.
Recently, M\'emoli \cite{memoli2009spectral} introduced the {\em spectral Gromov-Wasserstein distance}, applying the theory of mass transportation.
The spectral Gromov Wasserstein distance via the comparison of heat kernels satisfies all properties of a metric on the
class of isometric manifolds.

Here, we address the problem of shape correspondence in the context of shape similarity.
We find that comparing quasi-conformal quantities, provides a solid initial set of matched points.
Finer point-wise correspondence between shapes is detected by the matching of their respective spectral kernels.
Finally, accurate dense correspondence can be achieved by the simultaneous evaluation of the spectral kernel and the spectral embedding distances.

At a high level, the {\em spectral kernel} based approach substitutes the reality of geodesic distances on the surface
  with the notion of {\em spectral connectivity}.
Many options for spectral connectivity can be considered, depending on the underlying application.
One possibility is to use the heat kernel that provides a natural notion of scale, which is useful for multi-scale shape comparison.
Spectral distances, such as the diffusion distance \cite{bronstein2010gromov}, 
or the interpolated geodesic distances \cite{aflalo2013smds}, 
are alternative options.
In that perspective, the spectral kernel approach can be interpreted as a generalization of the spectral GMDS \cite{aflalo2013sgmds} framework.
For the task of point-wise correspondence and based on the data sets we analyzed, our empirical observations indicate 
 that the commute time kernel (a.k.a GPS kernel) \cite{qiu2007clustering,rustamov2007laplace}, 
 makes a prominent choice, as it well adapts to large isometric deformations exhibited in the data sets we analyzed.

Moreover, we try to bridge the gap between area-preserving and angle-preserving mappings.
It is well known that conformal mappings are similarities in the small in
the sense that over small regions, the image differs from the original
by a scale factor alone \cite{gu2004genus}. 
However, non-infinitesimal areas of the surface may be considerably enlarged or reduced.
Here, we try to account for these seemingly conflicting constraints, endorsing a spectral embedding design, that includes a natural balance between area-preserving and conformal requirements. 

\subsection {Contribution}
Our main contribution is a novel formulation and practical evaluation of shape similarity measures.
We contemplate different objective functions, based on the following types of spectral distances.
\begin{itemize}
\item The spectral kernel distance.
\item The spectral embedding distance.
\item The spectral quasi-conformal distance.
\end{itemize}
We introduce three optimization procedures for evaluating shape similarity.
\begin{itemize}
	\item First, {\em Spectral Quasi-Conformal Maps} that uses matched eigenfunctions for the estimation of the spectral quasi-conformal distance.
	\item Second, the {\em Spectral Kernel Maps} that evaluates the spectral kernel distance.
	\item Third, the {\em functional Spectral Kernel Maps} that uses functional representation for the evaluation of the spectral embedding distance and the spectral kernel distance.
\end{itemize}
Furthermore, we present a fully automatic framework for correspondence detection, that is robust to isometric deformations and achieves state-of-the-art performance.
We emphasize  that the spectral embedding metric and the spectral kernel distance express different
 attributes of shape structure.  
At one end, spectral embedding presents excellent point identification qualities. 
At the other end, comparing spectral kernels reflects the geometric connectivity of points on the surfaces.
The functional spectral kernel maps algorithm alternates between kernel distance evaluation and spectral embedding, 
 thus,  combines the merits of both. 

\subsection {Related work}

\subsection{Spectral kernels}


\subsubsection{Laplace-Beltrami eigendecomposition and the heat kernel}
\label{sec:lbo}
Let $X$ be a compact two-dimensional Riemannian manifold. 
The divergence of the gradient of a function $f$ over the manifold,
\begin{eqnarray*}
	\Delta f &\equiv& \mathrm{div}\, \mathrm{grad} f \nonumber
\end{eqnarray*}
is called the \emph{Laplace-Beltrami operator} (LBO) of $f$ and can be considered as a generalization of the standard notion of the Laplace operator to compact Riemannian manifolds \cite{beltrami1864ricerche,taubin1995signal,levy2006laplace}.
The Laplace-Beltrami  eigenvalue problem is given by
\begin{eqnarray*}
	-\Delta \phi_i &\equiv& \lambda_i \phi_i \text. \nonumber
\end{eqnarray*}
By definition, $- \Delta$ is a positive semidefinite operator.
Therefore, its eigendecomposition consists of non-negative eigenvalues $\lambda_i$.
We can order the eigenvalues as follows $0 = \lambda_0 < \lambda_1 < \cdots < \lambda_i < \cdots$ \quad .
The set of corresponding eigenfunctions given by $\{\phi_0, \phi_1, \cdots, \phi_i, \cdots\}$ forms an orthonormal basis, such that $\int_X \phi_i(x)\phi_j(x) d\text{V}_X = \delta_{ij}$, where $d\text{V}_X$ is the volume element on manifold $X$.

The solution of the heat equation $\cfrac{\partial u}{\partial t} = \Delta u$, with point heat source at 
 $x \in X$ (heat value at point $x' \in X$ after time $t>0$) is the heat kernel $K_t(x, x')$.
The heat kernel can be represented in terms of the Laplace-Beltrami eigenbasis as \cite{milgram1951harmonic}
\begin{eqnarray*}
	K_t(x, x') &=& \sum\limits_i e^{-\lambda_i t}\phi_i(x)\phi_i(x'). \nonumber
\end{eqnarray*}
The {\em heat kernel embedding} $\Phi_t: x \mapsto \ell^2$ maps the shape into the infinite dimensional spectral space 
\begin{eqnarray*}
	\Phi_t(x) &\equiv& \Phi(x)e^{-\Lambda t/2},	\nonumber
\end{eqnarray*}
where
\begin{eqnarray*}
	\Phi(x)&\equiv&\{\phi_1(x),\phi_2(x),\dots,\phi_i(x),\dots\},
	\cr \cr \Lambda &\equiv&\mbox{diag}(\lambda_1,\lambda_2,\dots,\lambda_i,\dots). \nonumber
\end{eqnarray*}
Therefore, the heat kernel can be represented as the inner product of the heat kernel embedding
\begin{eqnarray*}
	K_t(x, x') &=&  \Phi_t(x)  \cdot \Phi_t(x')\text. \nonumber
\end{eqnarray*}


\subsubsection{The GPS kernel}
\label{sec:gps}
Similarly to embedding into the heat kernel, Rustamov \cite{rustamov2007laplace} proposed the 
 {\em Global Point Signature} (GPS) embedding into
\begin{eqnarray*}
	\Phi_{\text{GPS}}(x) \eqs \equiv \eqs \left \{ \cfrac{1}{\sqrt{\lambda_1}}\,\phi_1(x),\cfrac{1}{\sqrt{\lambda_2}}\,\phi_2(x),\dots \right \}
	 \eqs = \eqs \Phi(x)\Lambda^{-\frac{1}{2}}
	 \text, \nonumber
\end{eqnarray*}
with the respective GPS kernel
\begin{eqnarray*}
	K_{\text{GPS}}(x, x') \eqs \equiv \eqs \Phi_{\text{GPS}}(x)  \cdot \Phi_{\text{GPS}}(x')  \eqs = \eqs \sum\limits_i\cfrac{1}{\lambda_i}\phi_i(x)\phi_i(x') \text. \nonumber
\end{eqnarray*}
The choice of the GPS kernel  for embedding is motivated in a number of ways.
\begin{itemize}
\item Rustamov \cite{rustamov2007laplace} argued that the GPS kernel coincides with the
   Green's function,
   and that the Green's function, in some sense, measures the extent to which two points are
    geometrically connected. 
\item It is easy to see that $K_{\text{GPS}}(x, x') = \int_{t=0}^\infty K_t(x, x')dt$, \ie the GPS kernel is the integration through all scales of the heat kernel \cite{bronstein2011shape}.
\item The GPS kernel is equivalent to the commute time  kernel \cite{qiu2007clustering}. 
    The commute time distance  can be thought of as the average time of a random walk starting 
     at $x$, passing through $x'$ and return back to $x$.
\item If ${\bar X}$ is obtained by uniformly scaling $X$ by a factor $\alpha > 0$, the eigenvalues
of the Laplace-Beltrami operator are scaled according to $\bar{\lambda}_i = \alpha^{-2}\lambda_i$, and the corresponding ${L^2}$-normalized
eigenfunctions become $\bar{\phi}_i(x) = \alpha^{-1}\phi_i(x)$. Then, it is easy to see that the GPS kernel $\sum_i{\lambda_i}^{-1}\phi_i(x)\phi_i(x')$,
 is global scale invariant \cite{bronstein2011shape}.
\item The gradients of the GPS coordinates are normalized. By definition $\Delta \phi_i(x) = -\lambda_i \phi_i(x)$ and $\int_X \phi^2_i(x)d\text{V}_X=1$. Applying Green's first identity on Riemannian manifolds
\begin{eqnarray*}
	\int_X \norm{\nabla\cfrac{ \phi_i(x)}{\sqrt{\lambda^{}_i}}}^2d\text{V}_X \eqs =  \eqs -\int_X \cfrac{\phi_i(x)}{\sqrt{\lambda^{}_i}} \Delta \cfrac{\phi_i(x)}{\sqrt{\lambda^{}_i}}\,d\text{V}_X \eqs = \eqs 1. \nonumber
\end{eqnarray*}

\end{itemize}

\subsubsection{Functional maps}
\label{sec:fun_map}
Ovsjanikov \etal \cite{ovsjanikov2012functional} presented the idea of functional representation.
Let $\varphi: X \mapsto Y$  be a bijective mapping between shapes $X$ and $Y$. 
If we are given a scalar function $f^X: X \mapsto \mathbb{R}$, then, we can 
 obtain a corresponding function $f^Y: Y \mapsto \mathbb{R}$ by the composition $f^Y=f^X \circ \varphi^{-1}$. 
Let us denote the induced transformation of generic space of real valued
  functions by $\varphi_F:\mathcal{F}(X,\mathbb{R}) \mapsto \mathcal{F}(Y,\mathbb{R})$, where we use $\mathcal{F}(\cdot,\mathbb{R})$
to denote a generic space of real-valued functions.
We call $\varphi_F$ the functional representation of the mapping $\varphi$.
Such a representation is linear, since for every pair of functions $f^X_1$,$f^X_2$ and scalars $\alpha_1$,$\alpha_2$, 
\begin{eqnarray*}
	\varphi_F(\alpha_1 f^X_1 + \alpha_2 f^X_2)  &=& (\alpha_1 f^X_1 + \alpha_2 f^X_2) \circ \varphi^{-1} 
	\cr &=& \alpha_1 f^X_1 \circ \varphi^{-1} + \alpha_2 f^X_2 \circ \varphi^{-1}  
	\cr &=& \alpha_1 \varphi_F(f^X_1) + \alpha_2 \varphi_F(f^X_2). \nonumber
\end{eqnarray*}
Next,
suppose that the function space of $X$ is equipped with a basis $\Phi^X = \{\phi^X_i\}$ so that any function
  $f^X : X \mapsto \mathbb{R}$ can be represented as a linear combination of basis functions $f^X = \sum_ia^X_i\phi_i^X$. 
If $Y$ is equipped with a basis $\Phi^Y = \{\phi^Y_i\}$, then $\varphi_F(\phi^X_i) = \sum_j ({C^X})_{ij}\phi^Y_j $ for some $(C^X)_{ij}$ and
\begin{eqnarray*}
	\varphi_F(f^X) \: = \: \varphi_F\bigg(\sum_i a^X_i\phi^X_i \bigg) \: = \: \sum_j \bigg(\sum_i a^X_i ({C^X})_{ij}\bigg) \phi^Y_j. \nonumber
\end{eqnarray*}
Equivalently, for $\varphi_F(f^X) = f^Y = \sum_ia^Y_i\phi_i^Y$, we have
\begin{eqnarray*}
	\varphi_F^{-1}(f^Y) \: = \: {\varphi_F}^{-1}\bigg(\sum_i a^Y_i\phi^Y_i \bigg) \: = \: \sum_j \bigg(\sum_i a^Y_i ({C^Y})_{ij}\bigg) \phi^X_j. \nonumber
\end{eqnarray*}
Therefore, we can represent $f^X$ as a row vector $\boldsymbol{a}^X$ with coefficients $a^X_i$, and equivalently, $f^Y$ as a row vector $\boldsymbol{a}^Y$ with coefficients $a^Y_i$. 
Then, we have ${\boldsymbol{a}^Y} = {\boldsymbol{a}^X}{C^X}$, and ${\boldsymbol{a}^X} = {\boldsymbol{a}^Y}{C^Y}$, where the matrices ${C^X}$ and ${C^Y}$ 
 are independent of $f^X$ and are completely determined by the bases $\Phi^X$ and $\Phi^Y $and the map $\varphi$. 
Here, we  restrict ourselves to use the first $N$ Laplace-Beltrami eigenfunctions as the 
 basis functions (see Section \ref{sec:lbo}) for the functional representations. 
This basis is well suited for representing near isometric shapes \cite{aflalo2013pca,ovsjanikov2012functional}.

Point-to-point correspondences assume that each point $x \in X$ corresponds to some point $y \in Y$ by the mapping $y = \varphi(x)$. 
In this case, the delta function $f^X(\tilde{x}) = \delta_{x}(\tilde{x})$ at point $x$ corresponds to the delta function $f^Y(\tilde{y}) = \delta_{y}(\tilde{y})$ at point $y$. We emphasize that $\tilde{x}$ and $\tilde{y}$ are the variables of the functions, while $x$ and $y$ are the respective parameters.
It is easy to show that for $\delta_{x}(\tilde{x})$ 
\begin{eqnarray*}
 	\boldsymbol{a}^X_x &=& \Phi^X(x) \eqs = \eqs (\phi^X_1(x), \phi^X_2(x), \dots,  \phi^X_i(x), \dots). \nonumber
\end{eqnarray*}
Equivalently, for $\delta_{y}(\tilde{y})$ 
\begin{eqnarray*}
	\boldsymbol{a}^Y_y &=& \Phi^Y(y) \eqs = \eqs (\phi^Y_1(y), \phi^Y_2(y), \dots, \phi^Y_i(y), \dots). \nonumber
\end{eqnarray*}
Then, we can construct the function preservation constraints $A^X = A^Y {C^Y}$, 
 where the matrices $A^X=\Phi^X$ and $A^Y=\Phi^Y$ are built by stacking the row vectors $\boldsymbol{a}^X_x$ 
 and $\boldsymbol{a}^Y_y$, respectively.
Thus, given some initial correspondence $y=\hat{\varphi}_{0}(x)$, a post-process iterative refinement
 algorithm can be obtained via \cite{ovsjanikov2012functional},

\begin{algorithm}[H]
\begin{algorithmic}
\item \caption{\textbf{:} \textsc{Post-process iterative refinement}}
\label{alg:ppir}
\For {$\ell=1$ to $L$}
\begin{enumerate}
	\item Construct the constraint matrices $A^X_\ell$, $A^Y_\ell$ by stacking $\boldsymbol{a}^X_x$, $\boldsymbol{a}^Y_y$, utilizing the correspondence provided by the previous iteration $y=\hat{\varphi}_{\ell-1}(x)$.
	\item \label{step:fm_const} Find the optimal $C^Y_\ell$ minimizing $\norm{A^X_\ell - A^Y_\ell C^Y_\ell}^2_F$.
	\item For each row $x$ of $\Phi^X$, find new correspondence $\hat{\varphi}_\ell(x)$ by searching for the closest row of $y$ in $\Phi^Y C^Y_\ell$.
\end{enumerate}
\EndFor
\end{algorithmic}
\end{algorithm}
The output of this algorithm is both a functional matrix ${C^Y}$ and point-wise correspondence.
This procedure is similar to the well known {\em Iterative Closest Point} (ICP) \cite{yang1992object,besl1992method} 
 in $N$ dimensions, except that it is done in the spectral domain, rather than the standard Euclidean
space. 

\section{Shape Similarity}

\subsection{Metric spaces and spectral distances}
In this section we first review known metrics between shapes modeled as 2-Riemannian manifolds. We then introduce a novel distance between manifolds which we denote by {\em spectral quasi-conformal distance}.
\subsubsection{Metric space}
A metric space is a pair $(X, d_X)$ where the set $X$ is equipped with 
 a distance $d_X:X \times X \mapsto \mathbb{R}_{\ge 0}$, satisfying the properties
\begin{enumerate}
	\item[P1:] Symmetry - $d_X(x, x') = d_X(x',x)$.
\label{met:prop1}
	\item[P2:] Triangle-inequality - $d_X(x, x') \le d_X(x, x'') + d_X(x'', x')$.
\label{met:prop2}
	\item[P3:] Identity - $d_X(x, x')=0$ if and only if $x$ = $x'$.
\label{met:prop3}
\end{enumerate}

\subsubsection{Gromov-Hausdorff distance}
Let $X$ and $Y$ be two compact metric spaces which are subsets of a common metric space $(Z, d_Z)$, 
 and we want to compare $X$ to $Y$. 
A common approach is that of computing \mbox{the {\em Hausdorff distance}} between them \cite{hausdorff1914grundzuege},
\begin{eqnarray*}
	d_\mathcal{H}^Z(X, Y ) &\equiv& \max(\sup_{x \in X} d_Z(x, Y), \sup_{y \in Y} d_Z(y, X)). \nonumber
\end{eqnarray*}
Now, if we allow the shapes to undergo isometric transformations, comparison is possible using the {\em Gromov-Hausdorff distance} \cite{gromov1981structures},
\begin{eqnarray*}
	d_\mathcal{GH}(X, Y) &\equiv& \inf_{Z,f,g}d_\mathcal{H}^Z(f(X), g(Y)), \nonumber
\end{eqnarray*}
where $f:X \mapsto Z$ and $g:Y \mapsto Z$ are isometric embeddings (distance preserving) into the metric space $Z$.
The  Gromov-Hausdorff distance can equivalently be defined as \cite{burago2001course}
\begin{eqnarray*}
	d_\mathcal{GH}(X, Y) &\equiv& \inf_{\begin{smallmatrix} \varphi: X \mapsto Y \\ \psi: Y \mapsto X \end{smallmatrix}}\cfrac{1}{2}\max(\text{dis}(\varphi), \text{dis}(\psi), \text{dis}(\varphi, \psi)), \nonumber
\end{eqnarray*}
where
\begin{eqnarray*}
	\text{dis}(\varphi)&\equiv&\sup_{x,x' \in X}\abs{d_X(x,x')-d_Y(\varphi(x),\varphi(x'))},
	\cr \text{dis}(\psi)&\equiv&\sup_{y,y' \in Y}\abs{d_X(\psi(y),\psi(y'))-d_Y(y,y')}, 
	\cr \text{dis}(\varphi,\psi)&\equiv&\sup_{x \in X, y \in Y}\abs{d_X(x,\psi(y))-d_Y(\varphi(x),y)}. \nonumber
\end{eqnarray*}

\subsubsection{Spectral embedding distance}
\label{par:spec_emb}
From the point of view of spectral geometry, B{\'e}rard \etal \cite{berard1994embedding}  used the spectral properties of the heat operator $e^{t\Delta}$ to define a metric between two Riemannian manifolds $X,Y \in \mathcal{M}$. Here, $\mathcal{M}$ is the set of all closed (\ie compact without boundary) Riemannian manifolds of dimension $n$.

Given a Riemannian manifolds $X \in \mathcal{M}$ with volume $\text{Vol}(X)$ and some $t > 0$, they based their metric on the eigendecomposition of heat kernel $K_t(x,x')$ and defined the spectral embedding $I^{\Phi^X}_t(x) = \{\sqrt{\text{Vol}(X)}e^{-\lambda^X_i t/2}\phi^X_i(x)\}_i$ by utilizing the eigenfunctions $\phi^X_i$ and eigenvalues $\lambda^X_i$ of the  Laplace-Beltrami operator $\Delta^X$.

Given a pair of eigenbases $\Phi^X$, $\Phi^Y$, they embedded the two Riemannian manifolds into their respective heat kernels $I^{\Phi^X}_t$, $I^{\Phi^Y}_t$, and measured the Hausdorff distance $d_\mathcal{H}^{\text{EMB}}$ between the manifolds in the common Euclidean space. 
\begin{eqnarray*}
	d_{\mathcal H}^{\text{EMB}}(I^{\Phi^X}_t,I^{\Phi^Y}_t) &\equiv& \inf_{\begin{smallmatrix}\varphi: X \mapsto Y \\ \psi: Y \mapsto X \end{smallmatrix}}\max(\text{dis}_t^{\text{EMB}}(\varphi), \text{dis}_t^{\text{EMB}}(\psi)),
	\cr \text{dis}_t^{\text{EMB}}(\varphi)&\equiv&\sup_{\begin{smallmatrix} x \in X\end{smallmatrix}}{d(I^{\Phi^X}_t(x),I^{\Phi^Y}_t(\varphi(x)))},
	\cr  \text{dis}_t^{\text{EMB}}(\psi)&\equiv&\sup_{\begin{smallmatrix} y \in Y\end{smallmatrix}}{d(I^{\Phi^X}_t(\psi(y)),I^{\Phi^Y}_t(y))}. \nonumber
\end{eqnarray*}
The distance between point $x \in X$ and $y \in Y$ is the usual Euclidean distance in the embedded space
\begin{eqnarray*}
	d(I^{\Phi^X}_t(x),I^{\Phi^Y}_t(y))^2 &=& \norm{I^{\Phi^X}_t(x) - I^{\Phi^Y}_t(y)}_{L^2}. \nonumber
\end{eqnarray*}
B{\'e}rard \etal defined the distance between the manifolds $X$, $Y$, as the upper-bound of the Hausdorff distance evaluated for all compatible pairs of eigenbases.
\begin{eqnarray}
\label{eq:spec_emb}
	d^{\text{EMB}}_t(X, Y) &\equiv& \max(d_{I_t}^{\text{EMB}}(X,Y), d_{I_t}^{\text{EMB}}(Y,X)), 
	\cr
	 \cr d_{I_t}^{\text{EMB}}(X,Y) &\equiv& \sup_{\{\Phi^X\}} \inf_{\{\Phi^Y\}} d_\mathcal{H}^{\text{EMB}}(I^{\Phi^X}_t,I^{\Phi^Y}_t),
	\cr d_{I_t}^{\text{EMB}}(Y,X) &\equiv&\sup_{\{\Phi^Y\}} \inf_{\{\Phi^X\}} d_\mathcal{H}^{\text{EMB}}(I^{\Phi^X}_t,I^{\Phi^Y}_t). 
\end{eqnarray}
We denote $d^{\text{EMB}}_t(X, Y) $ as the {\em spectral embedding distance}.

They showed that for any fixed $t>0$, the spectral embedding distance $d^\text{EMB}_t$ is a metric between isometry classes of Riemannian manifolds. In particular, $d^\text{EMB}_t(X,Y)=0$ if and only if the Riemannian manifolds $X$ and $Y$ are isometric.
We note that the theory holds for any kernel of the form $K(x, x') = \sum_i k(\lambda_i)\phi_i(x)\phi_i(x')$, with $k(\cdot)$ injective and decreasing sufficiently fast at infinity.

\subsubsection{Spectral kernel distance}
\label{par:spec_ker_dist}
Kasue and Kumura \cite{kasue1994spectral} defined the metric between Riemannian manifolds by comparing their respective heat kernels 
\begin{eqnarray}
\label{eq:spec_ker_dist}
	d^{\text{SPEC}}(X, Y) &\equiv& \inf_{\begin{smallmatrix}\varphi: X \mapsto Y \\ \psi: Y \mapsto X \end{smallmatrix}}\max(\text{dis}^{\text{SPEC}}(\varphi), \text{dis}^{\text{SPEC}}(\psi)), 
\end{eqnarray}
taking the supremum of kernel distortion for all $t>0$

\begin{eqnarray*}
	\text{dis}^{\text{SPEC}}(\varphi) &\equiv& \sup_{\begin{smallmatrix} x,x' \in X , t>0 \end{smallmatrix}}u(t)d_t^{\text{SPEC}}(x,x',\varphi(x),\varphi(x')),
	\cr \text{dis}^{\text{SPEC}}(\psi) &\equiv& \sup_{\begin{smallmatrix} y,y' \in Y , t>0 \end{smallmatrix}}u(t)d_t^{\text{SPEC}}(\psi(y),\psi(y'),y,y'), \nonumber
	\cr d_t^{\text{SPEC}}(x,x',y,y') &\equiv& \abs{\text{Vol}(X)K_t^X(x,x')-\text{Vol}(Y)K_t^Y(y,y')}. \nonumber
\end{eqnarray*}

The function $u(t)\equiv e^{-(t+1/t)}$ is used to normalize the kernels for different values of $t$. 
We denote $d^{\text{SPEC}}(X, Y) $ as the {\em spectral kernel distance}.

\subsubsection{Spectral Gromov-Wasserstein distance}
\label{par:spec_gw_dist}
Recently, M\'emoli \cite{memoli2009spectral} introduced a spectral version of the Gromov-Wasserstein distance, 
 similar to the one proposed by Kasue and Kumura. 
He applied the well-established transportation distances, also known as Earth Mover's Distance (EMD). 

Given importance probability measures $\mu^X(x)$, $\mu^Y(y)$ such that $\int_X\mu^X(x)dx=1$, $\int_Y\mu^Y(y)dy=1$, respectively,
 and for a coupling measure $\mu(x,y)$ satisfying $\int_X\mu(x,y)dx = \mu^Y(y)$ and $\int_Y\mu(x,y)dy = \mu^X(x)$, 
 the {\em spectral Gromov-Wasserstein distance} $d_{\mathcal{GW},p}(X,Y)$ is the EMD worst case $L^p$ norm of the spectral kernel distortion for all scales $t>0$,
\begin{eqnarray*}
	d_{\mathcal{GW},p}^{\text{SPEC}}(X, Y) &\equiv& \inf_{\mu}\sup_{t>0}c(t)\norm{\Gamma^{\text{SPEC}}_{t,X,Y}}_{L^p \, (\mu\otimes\mu)},  	
\cr \norm{\Gamma^{\text{SPEC}}_{t,X,Y}}^p_{L^p \, (\mu\otimes\mu)} &\equiv& \iint\limits_{\begin{smallmatrix} X \times Y \\ X \times Y\end{smallmatrix}}d_t^{\text{SPEC}}(x,x',y,y')^pd\mu(x,y)d\mu(x',y'), 
\end{eqnarray*}
where $c(t)\equiv e^{-a/t}$ for some $a>0$.

\subsubsection{Spectral quasi-conformal distance}
\label{par:spec_qc}
Motivated by conformal type distances \cite{gu2004genus,lipman2011conformal,aflalo2013conformal} and spectral embedding \cite{berard1994embedding},
we propose the {\em quasi-conformal distance}. The key point in our semi-conformal approach is that of replacing the angles between curves by analogous point-wise signatures. To that end, we regard the isometry invariant quantities 
\begin{eqnarray*}
\omega_{i,j} &\equiv& \cfrac{\nabla\phi_i \cdot \nabla\phi_j}{\sqrt{\lambda_i \lambda_j}}, \nonumber
\end{eqnarray*}
 which relate both to angles and areas of the triangles constructed by the gradients of the eigenfunctions $\nabla\phi_i$, $\nabla\phi_j$. Now, similarly to the spectral embedding, we define the embedding
\begin{eqnarray*}
	\tilde{I}^{\Phi^X}_t(x) \equiv \{\text{Vol}(X)e^{-(\lambda^X_i+\lambda^X_j) t/2}\omega^X_{i,j}(x)\}_{i,j}.
\end{eqnarray*}

\begin{theorem}
\label{thm:sep}
	The embedding $\tilde{I}^{\Phi^X}_t$ separates points on the Riemannian manifold $X \in \mathcal{M}$, \ie $x\ne x' \iff \tilde{I}^{\Phi^X}_t(x) \ne \tilde{I}^{\Phi^X}_t(x')$.
\end{theorem}

\begin{proof}
		Suppose that for two distinct points $x, x' \in X$ , we have $\tilde{I}^{\Phi^X}_t(x) = \tilde{I}^{\Phi^X}_t(x')$. This means that  for all $i, j$ we have $\nabla^X \phi^X_i(x) \cdot \nabla^X \phi^X_j(x) = \nabla^X \phi^X_i(x') \cdot \nabla^X \phi^X_j(x')$.
		Any scalar function $f^X:X \mapsto \mathbb{R}$  can be represented as a linear combination of the eigenbasis $f^X(x) = \sum_i a_i \phi_i^X(x)$.  We can write the norm of the gradient as
		\begin{eqnarray*}
			\norm{\nabla^X f^X(x)}^2 &=& \sum_{i,j} a_i a_j (\nabla^X \phi^X_i(x) \cdot \nabla^X \phi^X_j(x))
			\cr &=& \sum_{i,j} a_i a_j (\nabla^X \phi^X_i(x') \cdot \nabla^X \phi^X_j(x'))
			\cr  &=& \norm{\nabla^X f^X(x')}^2.  \nonumber
		\end{eqnarray*}
		However, one can easily imagine a  function $f^X$, such that $\norm{\nabla^X f^X}$ takes distinct values at those two points, a contradiction meaning that the $\tilde{I}^{\Phi^X}_t$ should have been different.
\end{proof}

Let $X$ and $Y$ be the two closed Riemannian manifolds we would like to compare, and let us be given some $t > 0$. We define the spectral quasi-conformal embedding $J^{\Phi^X}_t(x)  \equiv  \tilde{I}^{\Phi^X}_t(x) \cup I^{\Phi^X}_t(x)$ .
Given a pair of eigenbases $\Phi^X$, $\Phi^Y$, we embed the two Riemannian manifolds into  $J^{\Phi^X}_t$ and $J^{\Phi^Y}_t$, and measure the Hausdorff distance $d_\mathcal{H}^{\text{QC}}$ between the manifolds in the common Euclidean space. 
\begin{eqnarray*}
	d_{\mathcal H}^{\text{QC}}(J^{\Phi^X}_t,J^{\Phi^Y}_t) &\equiv& \inf_{\begin{smallmatrix}\varphi: X \mapsto Y \\ \psi: Y \mapsto X \end{smallmatrix}}\max(\text{dis}_t^{\text{QC}}(\varphi), \text{dis}_t^{\text{QC}}(\psi)),
	\cr \text{dis}_t^{\text{QC}}(\varphi)&\equiv&\sup_{\begin{smallmatrix} x \in X\end{smallmatrix}}{d(J^{\Phi^X}_t(x),J^{\Phi^Y}_t(\varphi(x)))},
	\cr  \text{dis}_t^{\text{QC}}(\psi)&\equiv&\sup_{\begin{smallmatrix} y \in Y\end{smallmatrix}}{d(J^{\Phi^X}_t(\psi(y)),J^{\Phi^Y}_t(y))}. \nonumber
\end{eqnarray*}
The distance between point $x \in X$ and $y \in Y$ is the usual Euclidean distance in the embedded space
\begin{eqnarray*}
	d(J^{\Phi^X}_t(x),J^{\Phi^Y}_t(y)) &=& \norm{J^{\Phi^X}_t(x) - J^{\Phi^Y}_t(y)}_{L^2}. \nonumber
\end{eqnarray*}
We define the {\em spectral quasi-conformal distance} as the supremum of the Hausdorff distance evaluated for all compatible pairs of eigenbases.
\begin{eqnarray}
	\label{eq:qc_dist}
	d^{\text{QC}}_t(X, Y) &\equiv& \max(d^{\text{QC}}_{J_t}(X,Y), d^{\text{QC}}_{J_t}(Y,X)), 
	\cr
	 \cr d_{J_t}^{\text{QC}}(X, Y) &\equiv& \sup_{\{\Phi^X\}} \inf_{\{\Phi^Y\}} d_\mathcal{H}^{\text{QC}}(J^{\Phi^X}_t,J^{\Phi^Y}_t),
	\cr d_{J_t}^{\text{QC}}(Y, X) &\equiv&\sup_{\{\Phi^Y\}} \inf_{\{\Phi^X\}} d_\mathcal{H}^{\text{QC}}(J^{\Phi^X}_t,J^{\Phi^Y}_t).
\end{eqnarray}

\begin{theorem}
\label{thm:si}
	The family $\{J^{\Phi^X}_t\}_t$ is invariant to global scaling of the metric.
\end{theorem}

\begin{proof}
If the Riemannian manifold ${\bar X}$ is obtained by uniformly scaling the metric of the Riemannian manifold $X$  by a factor $\alpha > 0$, then,
the eigenvalues of the Laplace-Beltrami operator are scaled according to $\lambda^{\bar X}_i = \alpha^{-2}\lambda^X_i$, the corresponding ${L^2}$-normalized eigenfunctions become $\phi^{\bar X}_i = \alpha^{-1}\phi^X_i$, and the gradient operator reads $\nabla^{\bar X} = \alpha^{-1}\nabla^X$. We have 
\begin{eqnarray*}
	\omega^{\bar X}_{i,j} &=& \cfrac{\nabla^{\bar X}\phi^{\bar X}_i \cdot \nabla^{\bar X}\phi^{\bar X}_j}{\sqrt{\lambda^{\bar X}_i \lambda^{\bar X}_j}}
	\eqs  =  \eqs  \cfrac{\alpha^{-1}\nabla^{X}\alpha^{-1}\phi^{X}_i \cdot \alpha^{-1}\nabla^{X}\alpha^{-1}\phi^{X}_j}{\sqrt{\alpha^{-2}\lambda^{X}_i \alpha^{-2}\lambda^{X}_j}}
	\cr &=&  \alpha^{-2}\cfrac{\nabla^{X}\phi^{X}_i \cdot \nabla^{X}\phi^{X}_j}{\sqrt{\lambda^{X}_i \lambda^{X}_j}} \eqs = \eqs  \alpha^{-2}\omega^{X}_{i,j}.
\end{eqnarray*}
Now, if we set $\bar{t} = \alpha^2 t$
\begin{eqnarray*}
	(\tilde{I}^{\Phi^{\bar X}}_{\bar t})_{i,j} &=& \text{Vol}({\bar X})e^{-(\lambda^{\bar X}_i+\lambda^{\bar X}_j) {\bar t}/2}\omega^{\bar X}_{i,j}
	\cr &=&\alpha^2\text{Vol}({X})e^{-(\alpha^{-2}\lambda^{X}_i+\alpha^{-2}\lambda^{X}_j) \alpha^2 t/2}\alpha^{-2}\omega^{X}_{i,j}\
	\cr &=& \text{Vol}({X})e^{-(\lambda^{X}_i+\lambda^{X}_j) t/2}\omega^{X}_{i,j} \eqs = \eqs(\tilde{I}^{\Phi^{X}}_{t})_{i,j},
	\cr
	\cr (I^{\Phi^{\bar X}}_{\bar t})_{i} &=& \text{Vol}({\bar X})^\frac{1}{2}\phi^{\bar X}_i e^{-\lambda^{\bar X}_i {\bar t}/2}
	\cr &=&(\alpha^{2}\text{Vol}({X}))^\frac{1}{2}e^{-\alpha^{-2}\lambda^{\bar X}_i \alpha^{2}t/2} \alpha^{-1}\phi^{X}_i
	\cr &=& \text{Vol}({X})^\frac{1}{2}e^{-\lambda^{X}_i {t}/2}\phi^{X}_i  \eqs = \eqs(I^{\Phi^{X}}_{t})_{i}.
\end{eqnarray*}
\end{proof}

\begin{theorem}
\label{thm:pmet}
For any fixed $t>0$, the spectral quasi-conformal distance $d^{\text{QC}}_t$ is a metric between isometry classes of Riemannian manifolds.
\end{theorem}
A complete proof is given in appendix \ref{app:qc}. 

\subsection{Shape similarity functionals}
First, we wish to define a more tractable version of the formal spectral kernel distance 
 that was presented in Eq. \ref{eq:spec_ker_dist}. 
In general, we want to find a direct mapping that preserves the spectral connectivity between points. 

Let us denote by $X$ and $Y$ the two shapes we compare. 
The correspondence between $X$ and $Y$ is represented by a bijective mapping $\varphi : X \mapsto Y,$ 
 such that for each point $x \in X$, the corresponding point $y \in Y$ is obtained by the mapping $y=\varphi(x)$. 
An importance measure denoted by $\rho^X(x)$ is specified for each point $x \in X$, such that $\int_X\rho^X(x)dx=1$.
 Based on the rationale of the spectral kernel distance, we wish to compare the kernels $K^X$ and $K^Y$.  
We define the {\em spectral kernel similarity functional}
\begin{eqnarray}
\label{eq:tract_spec_ker}	
	d_{p}^{\text{SPEC}}(X,Y) &\equiv& \min_{\varphi:X \mapsto Y}\norm{\Gamma^{\text{SPEC}}_{X,Y}}_{L^p}, 
\end{eqnarray}
that minimizes the $L^p$ norm of the spectral kernel distortion
\begin{eqnarray*}
	 \norm{\Gamma^{\text{SPEC}}_{X,Y}}^p_{L^p} &\equiv& \iint\limits_{X X}d_K^{\text{SPEC}}(x,x',\varphi(x),\varphi(x'))^pd\rho^X(x)d\rho^X(x'),
\end{eqnarray*}
where
\begin{eqnarray*}
	d_K^{\text{SPEC}}(x,x',y,y') &\equiv& \abs{K^X(x,x')-K^Y(y,y')}. \nonumber
\end{eqnarray*}
We remark that the definition of the spectral kernel similarity functional can easily be extended for comparing multiple kernels.
Here, we choose the kernel to be the GPS one.

Next, we would like to reformulate the spectral embedding distance, given in Eq. \ref{eq:spec_emb},
 as the objective function of the 
 iterative refinement algorithm defined in Section \ref{sec:fun_map}.
The iterative algorithm tries to fit the shapes embedded in a common Euclidean space, in such a way that a point on one shape matches to the closest point on the other shape.
Recall, that the correspondence $\varphi(x)$ of point $x \in X$ is found by searching for the nearest neighbor of $\Phi^X(x)$ in $\Phi^Y{C^Y}$.
We can think of $\Phi^X$ as an embedding of shape $X$ into the $N$ dimensional spectral space. 
Similarly, we can think of $\Phi^Y{C^Y}$ as a rigid transformation of the embedding $\Phi^Y$ 
 of shape $Y$ into the same spectral space.
Hence, we can present the {\em spectral embedding similarity functional}
\begin{eqnarray}
\label{eq:tract_se}	
	d_{p}^{\text{EMB}}(X,Y) &\equiv& \min_{C^Y,\varphi}\norm{\Gamma^{\text{EMB}}_{X,Y}}_{L^p}, 
\end{eqnarray}
that minimizes the $L^p$ norm of the spectral distance 
\begin{eqnarray*}
	\norm{\Gamma^{\text{EMB}}_{X,Y}}^p_{L^p} &\equiv& \int_X d^{\text{EMB}}(\Phi^X(x), \Phi^Y(\varphi(x)){C^Y})^pd\rho^X(x).\nonumber
\end{eqnarray*}
The metric $d^{\text{EMB}}$
defines the spectral distance between points in the common space.
Here, we choose the global scale invariant distance
\begin{eqnarray*}
	d^{\text{EMB}}(\Phi^X(x), \Phi^Y(y){C^Y}) &\equiv& \norm{(\Phi^X(x) - \Phi^Y(y){C^Y})\Lambda_X^{-\frac{1}{2}}}_{L^2}. \nonumber
\end{eqnarray*}

Finally, we derive an objective function from the spectral quasi-conformal distance of Eq. \ref{eq:qc_dist}. In addition to the angle and area-preserving quantity $\omega_{i,j}$, we incorporate the orientation of the shapes by introducing
\begin{eqnarray*}
\nu_{i,j} \eqs \equiv \eqs \cfrac{\textbf n \cdot (\nabla\phi_i \times \nabla\phi_j)}{\sqrt{\lambda_i \lambda_j}},\nonumber
\end{eqnarray*}
where $\textbf n$ is the outward pointing normal to the surface.
The point-wise signatures $\omega_{i, j}$ , $\nu_{i, j}$ can be related to the areas of the triangles constructed by the norms of the gradients $\norm{\nabla {\phi}_i}$, $\norm{\nabla {\phi}_j}$ and the angles between them $\angle(\nabla {\phi}_i, \nabla {\phi}_j)$. 
 Using  $\omega_{i, j}$ and $\nu_{i, j}$, we can formulate the {\em spectral quasi-conformal similarity functional}
\begin{eqnarray}
\label{eq:tract_cnf}	
	d_{p}^{\text{QC}}(X,Y) &\equiv& \sup\limits_{\{{\Phi}^X\}} \inf_{\{{\Phi}^Y\}}\inf_\varphi\norm{\Gamma^{\text{QC}}_{X,Y}}_{L^p}, 
\end{eqnarray}
 that optimizes the $L^p$ norm of the quasi-conformal distortion
\begin{eqnarray*}
	\norm{\Gamma^{\text{QC}}_{X,Y}}^p_{L^p} &\equiv& \int_X d^{\text{QC}}(x,\varphi(x))^p d\rho^X(x). \nonumber
\end{eqnarray*}
The quasi-conformal distortion $d^{\text{QC}}(x,y)$ between point $x \in X$ and point $y \in Y$  is defined using $\omega_{i, j}$ and $\nu_{i, j}$.
Here, we choose the $L^2$ distortion
\begin{eqnarray}
\label{eq:cnf_dist}	
	d^{\text{QC}}(x,y)^2 = \sum\limits_{i,j} (\omega^X_{i,j}(x) - \omega^Y_{i,j}(y))^2 + (\nu^X_{i,j}(x) - \nu^Y_{i,j}(y))^2.
\end{eqnarray}
Let us summarize. In this section we formulated three objective functions. 
\begin{itemize}
	\item The spectral kernel similarity functional $d_{p}^{\text{SPEC}}(X,Y)$.
	\item The spectral embedding similarity functional $d_{p}^{\text{EMB}}(X,Y)$.
	\item The spectral quasi-conformal similarity functional $d_{p}^{\text{QC}}(X,Y)$.
\end{itemize}
\section{Shape similarity evaluation}

\subsection{Spectral kernel maps}
\label{sec:skm}
We now address the issue of minimizing the spectral kernel similarity functional $d_{p}^{\text{SPEC}}(X,Y)$. For that goal we present a randomized iterative algorithm.
The key idea is that previously obtained correspondences are used to find a new mapping, and that at each iteration, new $M$ randomly selected points are being matched.

 Let $X$ and $Y$ be the discrete shapes we wish to compare.
We estimate the correspondence $\hat{\varphi}(x)$ by minimizing the $L^2$ discrete version of the spectral
  kernel similarity functional of Eq. \ref{eq:tract_spec_ker},
\begin{eqnarray*}
	\hat{\varphi} &=& \min_{\varphi}\sum_{\begin{smallmatrix}{x,x' \in X}\end{smallmatrix}}d_K^{\text{SPEC}}(x,x',\varphi(x),\varphi(x'))^2\rho_X(x)\rho_X(x'). \nonumber
\end{eqnarray*}
Suppose we are given  a subset of points $X_0 \subseteq X$ and an initial correspondence $\hat{\varphi}_0(x), x \in X_0$.
The Spectral Kernel Maps algorithm is given by

\begin{algorithm}[H]
\begin{algorithmic}
\item \caption{\textbf{:} \textsc{Spectral kernel maps}}
 \label{alg:skm}
\For {$\ell=1$ to $L$}
\begin{enumerate}
 \item Select a subset of points $X_\ell$ by drawing $M$ points from $X$ with probability $P(x)=\rho^X(x)$.
 \item For each $x \in X_\ell$ find  new correspondence $\hat{\varphi}_\ell(x)$ by minimizing 
\begin{eqnarray*}
	\hat{\varphi}_\ell(x)  =  \argmin\limits_{y \in Y} \sum_{x' \in X_{\ell-1}}\abs{K^X(x,x') - K^Y(y,\hat{\varphi}_{\ell-1}(x'))}^2. \nonumber
\end{eqnarray*}
\end{enumerate}
\EndFor
\end{algorithmic}
\end{algorithm}
The algorithm provides as an output the kernel signatures $K^X(x, x')$, $K^Y(y, \hat{\varphi}_L(x')) \:\, \forall \,  x' \in X_L$,
and the mapping
\begin{eqnarray*}
	\hat{\varphi}(x)&=&\argmin\limits_{y \in Y}{\sum\limits_{x' \in X_L}\abs{K^X(x,x') - K^Y(y,\hat{\varphi}_L(x'))}^2} \:\,  \forall \, x \in X. \nonumber
\end{eqnarray*}
We note that the spectral kernel maps are equivalent in structure to the
  {\em Heat Kernel Maps}  \cite{ovsjanikov2010one}, with the important difference that
   unlike the heat kernel maps, our design does not depend on accurate initial matching of  feature points. 
Moreover, the spectral kernel maps method can scale up the number of points being
 used for generating the isometric signatures, which consequently increases the robustness 
 and the accuracy of the mapping.

\subsection{Functional spectral kernel maps}
\label{sec:fskm}
We further extend the idea put forward in the previous section. In addition to optimizing the spectral kernel similarity objective function $d_{p}^{\text{SPEC}}(X,Y)$, we also want to minimize the spectral embedding similarity functional $d_{p}^{\text{EMB}}(X,Y)$.

Recall, that the post-process iterative refinement procedure presented in Algorithm \ref{alg:ppir}, already optimizes $d_{2}^{\text{EMB}}(X,Y)$. We are looking for a way to modify Algorthim \ref{alg:ppir}, in a manner that will simultaneously minimize $d_{2}^{\text{SPEC}}(X,Y)$, by matching $K^X(x,x')$ to $K^Y(y,y')$.
For that goal, we impose the spectral kernel constraints on the delta functions that represent correspondence. Accordingly, the function
\begin{eqnarray*}
  	f^X_{x,x'}(\tilde{x}) &=& ({K^X(x,x')}/{\abs{K^X(x,x)}})\delta_{x}(\tilde{x}), \nonumber
\end{eqnarray*}
 should correspond to 
\begin{eqnarray*}
	f^Y_{y,y'}(\tilde{y}) &=& ({K^Y(y,y')}/{\abs{K^Y(y,y)}})\delta_{y}(\tilde{y}). \nonumber
\end{eqnarray*}
For clarity we note that $\tilde{x}$ and $\tilde{y}$ are the variables of the functions, while $x,x'$ and $y,y'$ are the respective parameters.
Therefore, we have
\begin{eqnarray*}
	\boldsymbol{a}^X_{x,x'} &=& ({K^X(x,x')}/{\abs{K^X(x,x)}})\Phi^X(x), \nonumber
\end{eqnarray*}
 and equivalently
\begin{eqnarray*}
	\boldsymbol{a}^Y_{y,y'} &=& ({K^Y(y,y')}/{\abs{K^Y(y,y)}})\Phi^Y(y). \nonumber
\end{eqnarray*}
Then, we can construct the function preservation constraints $A^X = A^Y {C^Y}$, 
 where the matrices $A^X$, $A^Y$ are built by stacking the row vectors $\boldsymbol{a}^X_{x,x'}$ 
 and $\boldsymbol{a}^Y_{y,y'}$, respectively.
Notice that we normalize the kernels, so that the impact of all constraints $\boldsymbol{a}^X_{x,x}$, $\boldsymbol{a}^Y_{y,y}$  are similar for all $x \in X ,y \in Y$.

By recalling that for nearly isometric shapes, the correspondence we are looking for should be represented by a nearly-diagonal ${C^Y}$ \cite{kovnatsky2012coupled}, we can submit an element-wise off-diagonal penalty $W$ and formulate the following problem
\begin{eqnarray}
\label{eq:off_diag}
	\argmin_{{C^Y}} \norm{A^X - A^Y{C^Y}}^2_F + \beta\norm{W\odot {C^Y}}^2_F,
\end{eqnarray}
where $\beta$ is a tuning parameter.
The minimization of  (\ref{eq:off_diag}) can be obtained separately
 for each column of $C^Y$ by the least squares method.

Suppose we are given  a subset of points $X_0 \subseteq X$ and an initial correspondence $\hat{\varphi}_0(x), \:\, \forall \,  x \in X_0$.
We are ready to  present the functional Spectral Kernel Maps algorithm

\begin{algorithm}[H]
\begin{algorithmic}
\item \caption{\textbf{:} \textsc{Functional spectral kernel maps}} 
\label{alg:fskm}
\For {$\ell=1$ to $L$}
\begin{enumerate}
	\item Construct the constraint matrices $A^X_\ell$, $A^Y_\ell$ by stacking $\boldsymbol{a}^X_{x,x'}$, $\boldsymbol{a}^Y_{y,y'}$, utilizing the correspondence provided by the previous iteration $y=\hat{\varphi}_{\ell-1}(x)$, $y'=\hat{\varphi}_{\ell-1}(x')$.
	\item \label{step:SKFM_c} Find the optimal $C^Y_\ell$ minimizing $\norm{A^X_\ell - A^Y_\ell C^Y_\ell}^2_F+\beta\norm{W\odot {C_\ell^Y}}^2_F$.
	 \item Select a subset of points $X_\ell$ by drawing $M$ points from $X$ with probability $P(x)=\rho^X(x)$.
	\item \label{step:SKFM_closest} For each row $x \in X_\ell$ of $\Phi^X$, find new correspondence $\hat{\varphi}_\ell(x)$ by searching for the closest row of $y$ in $\Phi^Y C^Y_\ell$, applying $\hat{\varphi}_\ell(x) = \argmin\limits_{y\in Y} d^{\text{EMB}}(\Phi^X(x), \Phi^Y(y)C^Y_\ell)$.
\end{enumerate}
\EndFor
\end{algorithmic}
\end{algorithm}
The  algorithm provides as an output both the functional matrix ${C^Y}$ computed in Step \ref{step:SKFM_c},
 and the point-wise correspondence $\hat{\varphi}(x)$ found in Step \ref{step:SKFM_closest}.

We note that the functional spectral kernel maps algorithm is closely related to the spectral GMDS procedure  \cite{aflalo2013sgmds}, 
 as both try to match the connectivity between points on each surface by functional representation.
It can also be viewed as a generalization of the post-process iterative refinement algorithm.
This is noticed by setting the kernel $K(x,x')$ to be the heat kernel $K_t(x,x')$. As $t \to 0$  the only constraints that remain are $\boldsymbol{a}^X_{x,x} \to \Phi^X(x)$ and $\boldsymbol{a}^Y_{y,y} \to \Phi^Y(y)$.
The advantage of the functional spectral kernel maps algorithm is in the combination of the two methods, resulting in highly accurate 
 correspondence maps.

\subsection{Spectral quasi-conformal maps}
\label{sec:conf}
We now turn to approximate correspondence via the spectral {quasi-conformal} similarity functional $d_{p}^{\text{QC}}(X,Y)$, given in Eq. \ref{eq:tract_cnf}. Let $\{\phi^X_i\}_{i=1}^{N_0}$, $\{\phi^Y_i\}_{i=1}^{N_0}$ be the first $N_0$ eigenfunctions of the shapes $X$ and $Y$, respectively. We suggest to find correspondence for each point $x \in X$ by minimizing the quasi-conformal distortion of Eq. \ref{eq:cnf_dist}, introducing the \mbox{\em Spectral Quasi-Conformal Maps}
\begin{eqnarray*}
	\hat{\varphi}(x)=\argmin_{y \in Y_x} \sum\limits_{i,j=1}^{N^0}  (\omega^X_{i,j}(x)-\omega^Y_{i,j}(y))^2 + (\nu^X_{i,j}(x) - \nu^Y_{i,j}(y))^2. \nonumber
\end{eqnarray*}

The correspondence $\hat{\varphi}(x)$ is chosen from a limited set of points $y\in Y_x$. 
The reduced subset of potential points $Y_x$ is obtained by comparing the point-wise signatures $\Psi^X$ to $\Psi^Y$.
In Subsection \ref{sec:mes} we explain how to obtain the point-wise signature $\Psi$. For each point $x \in X$ we find $\abs{Y_x}$ points $y \in Y$ that are closest to $x$ in sense of to $L^2$ distance $\norm{\Psi^X-\Psi^Y}_{L^2}$. We note that this simple preprocessing step boosts the performance of the quasi-conformal maps.

We point out, that the quantities $\omega^X_{i,j}$ and $\nu^X_{i,j}$ are meaningful only if we have sets of corresponding eigenfunctions. In the next section we describe how this condition is fulfilled by a robust eigenfunction matching technique.

\subsubsection{Robust eigenfunction matching}
\label{sec:ef_matching}
The need to optimize cost functions over all orthonormal pairs of eigenbases, arises regularly while evaluating spectral similarity of shapes, like the spectral embeding distance and the spectral quasi-conformal distance.
Computationally, this exhaustive search can be relaxed by an indirect approach that first finds the orthonormal eigenbasis $\Phi^Y$, that best corresponds to $\Phi^X$.

Given orthonormal eigenfunctions $\tilde{\Phi}^Y = \{\tilde{\phi}^Y_i\}$ of a generic shape (\ie one without repeated eigenvalues), we can derive the family of orthonormal eigenbases $\Phi^Y=\{s_i\tilde{\phi}^Y_{\pi_i}\}$ by modifying the signs $s_i \in \{+1, -1\}$ of the eigenfunctions $\tilde{\phi}_i^Y$, and reordering them by $\pi_i \in \mathbb{Z}_{> 0}$. For approximately isometric generic shapes $X$ and $Y$, we would like to find the optimal signs $\boldsymbol{s}=\{s_i\}$  and permutations $\boldsymbol{\pi}=\{\pi_i\}$ such that $\phi_i^X(x) \approx s_i\tilde{\phi}^Y_{\pi_i}(\varphi(x))$. To guarantee with high probability that we do not encounter eigenvalue multiplicity, we limit the search to a small number of compatible eigenfunctions. Next, we show how matching of the first $N^0$ eigenfunctions can be achieved, under the assumption that their respective eigenvalues are non-repeating.

\paragraph{Matching via high order statistics}
We adopt the approach presented in \cite{shtern2013hos} which showed that using {\em high order statistics} 
  can resolve the matching parameters. 
Here, we limit our method to comparing third order moments 
\begin{eqnarray*}
	\xi^X_{i,j,k} &\equiv& \int_X\phi_i\phi_j\phi_kda_X, \quad i,j,k \in \{1,2, \dots,N^0\}, \nonumber
\end{eqnarray*}
where $da_X$ is the area element of surface $X$.

Given two isometric shapes $X$ and $Y$, we can find the matching parameters $\boldsymbol{s}$ and $\boldsymbol{\pi}$ by minimizing the difference of the third order moments on the two shapes
\begin{eqnarray}
\label{eq:sign}
	\hat{\boldsymbol{s}}, \hat{\boldsymbol{\pi}} &=& \argmin\limits_{\boldsymbol{s},\boldsymbol{\pi}}\sum\limits_{i,j,k}(\xi^X_{i,j,k} - s_i s_j s_k\xi^Y_{\pi_i,\pi_j,\pi_k})^2.
\end{eqnarray}
This method considerably reduces the search space of possible corresponding eigenfunctions.
Alas, in the presence of intrinsic symmetry, the third order moments of the antisymmetric eigenfunctions are ambiguous and cannot be compared effectively.
Therefore, the signs of the antisymmetric eigenfunctions cannot be determined completely.
In the next section we remedy this situation, further evaluating the reduced search space of compatible eigenbases.

\paragraph{Matching via correspondence quality analysis}
The need to have a pair of compatible eigenbasis is of important for the initialization of the quasi-conformal maps procedure. After determining the permutation $\boldsymbol \pi$, and excluding some sign sequences, as described in the previous section, we further compare the eigenbases, by estimating the sign sequence $\boldsymbol s$, so that the eigenbasis $\Phi^Y(\boldsymbol s)=s_i\tilde{\phi}^Y_{\pi_i}$, best match the eigenbasis $\Phi^X$. Assuming for now, that we have a function $Q(\hat\varphi)$ which is maximized by the best mapping $\hat{\varphi}(\hat{\boldsymbol{s}})$, we can follow Algorithm \ref{alg:qa}.
\begin{algorithm}[H]
\begin{algorithmic}
\item \caption{\textbf{:} \textsc{\footnotesize Matching via correspondence quality analysis}}
\label{alg:qa}
\begin{itemize}
	\item For a possible sign sequence $\boldsymbol{s}$, estimate the mapping $\hat\varphi(\boldsymbol s)$ by applying a correspondence framework that uses $\boldsymbol{s}$ for initialization. 
	\item Evaluate the correspondence quality by the measure $Q(\hat{\varphi}(\boldsymbol s))$. 
	\item Search for the optimal sign sequence $\hat{\boldsymbol{s}}$ that maximizes $Q(\hat{\varphi}(\hat{\boldsymbol s}))$.
	\item Compute the compatible eigenbasis $\Phi^Y(\hat{\boldsymbol{s}})$.
\end{itemize}
\end{algorithmic}
\end{algorithm}
Next, we seek for a criterion $Q(\hat\varphi)$ that enfolds the consistency of the correspondence $\hat\varphi(\boldsymbol s)$.
For this purpose many heuristics can be considered. We have selected to maximize the {\em orientable area correlation}.

\paragraph{Orientable area correlation}
\label{app:triangle_transforms}
In this section we evaluate the quality of the correspondence $\varphi:X \mapsto Y$ from the point of view of area preservation criteria. Intuitively, for nearly isometric shapes and a good correspondence $\varphi$, the image of any triangle and the triangle itself should have approximately the same area.

Let $\tau^Y$  be a triangle with vertices $y_iy_j y_k$ on shape $Y$, having coordinates $V_{y_i}V_{y_j}V_{y_k}$ in $\mathbb{R}^3$.
The normal to the triangle ${\boldsymbol{n}}^Y(\tau^Y)$ can be estimated in two ways.
\begin{itemize}
	\item ${\boldsymbol{\hat{n}}}_1^Y(\tau^Y)$ - the cross product of the vectors $(V{y_j}-V{y_i})$ and $(V{y_k}-V{y_i})$
\begin{eqnarray*}
	\tilde{\boldsymbol{n}}^Y(\tau^Y) &\equiv& (V{y_j}-V{y_i}) \times (V{y_k}-V{y_i}),
	\cr {\boldsymbol{\hat{n}}}_1^Y(\tau^Y) &\equiv&{\tilde{\boldsymbol{n}}^Y(y) /\norm{\tilde{\boldsymbol{n}}^Y(y)}}. \end{eqnarray*}
	\item ${\boldsymbol{\hat{n}}}_2^Y(\tau^Y)$ - the average of the normals at the vertices
\begin{eqnarray*}
	{\boldsymbol{\hat{n}}}_2^Y(\tau^Y) &\equiv& \cfrac{1}{3}(\boldsymbol{n}^Y(y_i)+\boldsymbol{n}^Y(y_j)+\boldsymbol{n}^Y(y_k)), \nonumber
\end{eqnarray*}
	where $\boldsymbol{n}^Y(y)$ is the usual normal to the surface at point $y \in Y$.
\end {itemize}
We note that ${\boldsymbol{\hat{n}}}_1^Y(\tau^Y) \approx {\boldsymbol{\hat{n}}}_2^Y(\tau^Y)$ for a small enough triangle $\tau^Y$.

Let $T^X$ be a triangulation of the surface $X$. We order the vertices  $x_ix_j x_k$ of each triangle $\tau^X \in T^X$  in such a way that ${\boldsymbol{n}}_1^X(\tau^X)$ points outward. The triangle $\tau^X$ is mapped by $y_i=\varphi(x_i),y_j=\varphi(x_j),y_k=\varphi(x_k)$. We denote by $\tau_\varphi^Y$ the imaged triangle with vertices $y_iy_jy_k$.
For the special case where the shapes are isometric, and for a good map $\varphi(x)$, we expect that the area $\mathcal{A}(\tau^X)$ of an infinitesimal  triangle $\tau^X$ will be approximately equal to the area $\mathcal{A}(\tau_\varphi^Y)$ of its image $\tau_\varphi^Y$, and that ${\boldsymbol{\hat{n}}}_1^Y(\tau_\varphi^Y) \cdot {\boldsymbol{\hat{n}}}_2^Y(\tau_\varphi^Y) \approx 1$, as the two versions of the normal point to the same direction. In general, ${\boldsymbol{\hat{n}}}_1^Y(\tau_\varphi^Y)$ does not coincide with ${\boldsymbol{\hat{n}}}_2^Y(\tau_\varphi^Y)$, for example if the orientation is somehow flipped by the mapping $\varphi(x)$, then, the normals will have different directions and ${\boldsymbol{\hat{n}}}_1^Y(\tau_\varphi^Y) \cdot {\boldsymbol{\hat{n}}}_2^Y(\tau_\varphi^Y) \approx -1$. The normal  ${\boldsymbol{\hat{n}}}_1^Y(\tau_\varphi^Y)$ is in fact the image of the normal ${\boldsymbol{\hat{n}}}^X(\tau^X)$. 

The above discussion motivates us to measure the correlation between the area of $\tau^X$ oriented by ${\boldsymbol{\hat{n}}}_1^Y(\tau_\varphi^Y)$ to the area of $\tau_\varphi^Y$ oriented by ${\boldsymbol{\hat{n}}}_2^Y(\tau_\varphi^Y)$. We hereby define the {\em orientable area-preservation quality} $Q(\varphi)$ by

\begin{eqnarray}
\label{eq:oac}
	Q(\varphi)&\equiv& \cfrac{\sum\limits_{\tau^X \in T^X}
	 \big(\mathcal{A}(\tau^X){\boldsymbol{\hat{n}}}_1^Y(\tau_\varphi^Y)\big) \cdot \big(\mathcal{A}(\tau_\varphi^Y){\boldsymbol{\hat{n}}}_2^Y(\tau_\varphi^Y)\big)}
	{\sqrt{\sum\limits_{\tau^X \in T^X}{\mathcal{A}^2(\tau^X)}\sum\limits_{\tau^X \in T^X}{\mathcal{A}^2(\tau_\varphi^Y)}}}\,,
\end{eqnarray}
where the imaged triangle $\tau_\varphi^Y$ is constructed by the mapping $\varphi$ applied to the triangle $\tau^X \in T^X$.
\subsubsection{Matched eigenbasis signature}
\label{sec:mes}
As a preprocessing step for the spectral quasi-conformal maps computation, we would like to find an informative signature $\Psi$ that can be compared between the two shapes. Having a matched pair of eigenbases comprised of $N^0$ compatible eigenfunctions, shapes can be compared by a mix of global and local quantities.
\begin{enumerate}
	\item The matched $N^0$ low order eigenfunctions $\phi_i$ that capture the global structure of the shapes.
	\item The pseudo {\em Wave Kernel Signature} \cite{aubry2011wave} bandpass filters $\tilde\gamma_{t}(x) \equiv \sum\limits_{i=1}^N \lambda_i e^{-\lambda_i t}\phi^2_i(x)$, controlled by the parameter $t$, 
	     that better describe local features \cite{shtern2013hos}.
\end{enumerate}
Thus, we construct the matched eigenbasis signature
\begin{eqnarray*}
	\Psi(x) &\equiv& \{\phi_1(x),\phi_2(x),\cdots,\phi_{N_0}(x),
		\cr && {\gamma}_{t_1}(x),{\gamma}_{t_2}(x),\cdots,{\gamma}_{t_B}(x)\}, \nonumber
\end{eqnarray*}
combining together the $N^0$  low order eigenfunctions with the normalized bandpass filters 
\begin{eqnarray*}
{\gamma}_{t}(x) &\equiv& \frac{\tilde\gamma_{t}(x)}{\sqrt{\int_X{\tilde\gamma^2_{t}(x)da_X}}}, \nonumber
\end{eqnarray*}
sampled $B$ times at $t=\{t_1,t_2,\dots,t_B\}$.

\section {Correspondence framework}
\label{sec:corr_frm}
In this section we convert the observations outlined in the previous sections into a complete and applicable framework for correspondence detection.
The framework consists of the following sequential steps.
\begin{itemize}
\item Initialize coarse correspondence, using spectral quasi-conformal maps, see Section \ref{sec:conf}. 
\item Refine correspondence by applying the spectral kernel maps algorithm, see Section \ref{sec:skm}.
\item Dense correspondence is achieved by the functional spectral kernel maps with off-diagonal penalty, see Section \ref{sec:fskm}.
\end{itemize}
We note that a matched pair of eigenbases is found by using a combination of high order 
 statistics and correspondence quality analysis, see Section \ref{sec:ef_matching}.


\subsection{Implementation}
\label{sec:implementation}
The proposed correspondence framework is fully automatic.
As such, in all our experiments we used the same choice of parameters.
In Section \ref{sec:gps} we explained some interesting properties of the GPS kernel. Our empirical evidence suggests that the GPS kernel, that is the normalization of the eigenfunctions $\phi_i$ by the factor ${(\sqrt{\lambda_i})}^{-1}$, provides superior qualities for correspondence detection, compared to other kernels we tested. In general, we chose our parameters for achieving the most accurate results in a reasonable time. To that end, we used $N=120$ eigenfunctions of the Laplace-Beltrami operator.

For analysis, we break down the correspondence framework into different modules, with the following parameters.
\begin{itemize}
	\item Robust eigenfunction matching - see Section \ref{sec:ef_matching} , with the size of the triangulation set to $\abs{T^X} = 2000$. Mesh simplification is performed by the QSlim software \cite{garland1997qslim}.
	\item Matched eigenbasis signature - see Section \ref{sec:mes}, $N_0 = 6$ matched eigenfunctions and $B=6$ logarithmically sampled bandpass filters, with $t_1=\cfrac{1}{50\lambda_1}$ and $t_B = \cfrac{1}{\lambda_1}$. 
	\item Spectral quasi-conformal maps - see Section \ref{sec:conf}, with $N_0 = 6$ matched eigenfunctions, and $|Y_x|=100$.
	\item Spectral kernel maps - see Section \ref{sec:skm}, $K$ is the GPS kernel, $M=1000$ points and $L=15$ iterations. $P(x)$ uniformly selects points from the mesh.
	\item Functional spectral kernel maps - see Section \ref{sec:skm}, $K$ is the GPS kernel, $M=2000$ points and $L=15$ iterations. The off-diagonal penalty $W_{ij} = \cfrac{\abs{\lambda^Y_i-\lambda^{X}_j}}{\lambda^X_j} \,U_{i}$ was set to be proportional to the difference of the eigenvalues $\lambda^Y_i$ and $\lambda^X_i$, and scaled by the $i_\text{th}$ entry of the diagonal $U = \mbox{diag}((A^Y)^T A^Y)$, the tuning parameter $\beta$ was set to $0.1$.
\end{itemize}

\subsection{Results}
\label{sec:results}

We tested the proposed method on pairs of shapes represented by triangulated meshes from both the TOSCA  database \cite{bronstein2008numerical} and from the SCAPE database
 \cite{anguelov2005correlated}.
The TOSCA dataset contains densely sampled synthetic human and animal surfaces,
 divided into several classes with given ground-truth point-to-point correspondences between the shapes within each class.
 The SCAPE dataset contains scans of real human bodies in different poses.
We compare our work to several correspondence detection  methods.
\begin{itemize}
	\item {Spectral Kernel Maps} - the method proposed in this paper.
	\item {Functional Maps + Blended} (TOSCA only) - the functional maps based post-process iterative refinement algorithm. We use the results shown in \cite{ovsjanikov2012functional}. There, the post-process procedure refines the correspondence provided by the Blended method \cite{kim2011blended}.
	\item {Blended} - the method proposed by Kim \etal that uses a weighted combination of isometric maps \cite{kim2011blended}.
	\item {M{\"o}bius Voting} - the method proposed by Lipman \etal counts {\it votes} on the conformal M{\"o}bius transformations \cite{lipman2009mobius}.
	\item {Permuted Sparse Coding} + MSER (SCAPE only) - the approach proposed by Pokrass \etal finds correspondence by using methods from the field of sparse modeling. We note that this method depends on the ability to detect repeatable regions between shapes. We use the results shown in \cite{pokrass2012sparse}. There, maximally stable extremal regions (MSER) are used as a preprocessing step \cite{litman2011diffusion}.
	\item {Functional Maps + MSER} (SCAPE only) - the functional maps based post-process iterative refinement algorithm. We use the results shown in  \cite{pokrass2012sparse}.
\end{itemize}

Figure \ref{fig:TOSCA_Correspondences} compares our correspondence framework with existing methods on the TOSCA benchmark, using the evaluation protocol proposed in \cite{kim2011blended}. 
The distortion curves describe the percentage of surface points falling within
a relative geodesic distance from what is assumed to be their true locations. 
For each shape, the geodesic distance is normalized by the square root of the shape's area.
As is evident from the benchmark the proposed method significantly outperforms existing ones.

\begin{figure}[th]
	\center{\bf TOSCA correspondence}
	\begin{center}
	\begin{overpic}[width=1.0\columnwidth]{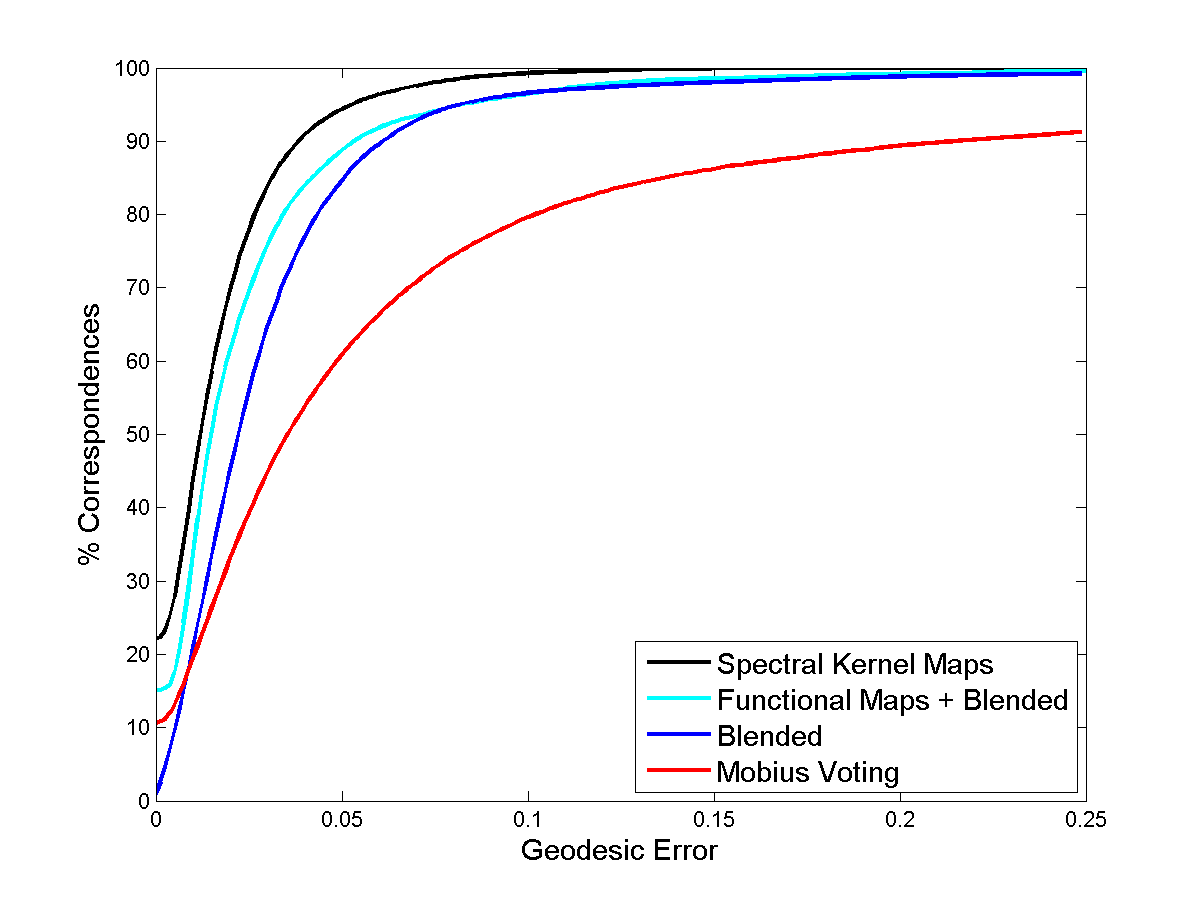}
	\end{overpic}
	\end{center}
	\caption{\small Evaluation of the spectral kernel maps algorithm applied to shapes from the 
	                TOSCA database, using the protocol of \cite{kim2011blended}.}
	\label{fig:TOSCA_Correspondences}
\end{figure}

Figure \ref{fig:SCAPE_Correspondences} compares the proposed correspondence framework with existing methods on the SCAPE database, 
again using the evaluation protocol proposed in \cite{kim2011blended}, allowing symmetries. 
Remark: In the evaluation, we do not allow per-point symmetry selection (as other methods do), but rather the correct symmetry is automatically chosen for the shape as a whole.

\begin{figure}[th]
	\center{{\bf SCAPE correspondence} \\ (allow symmetries)}
	\begin{center}
	\begin{overpic}[width=1.0\columnwidth]{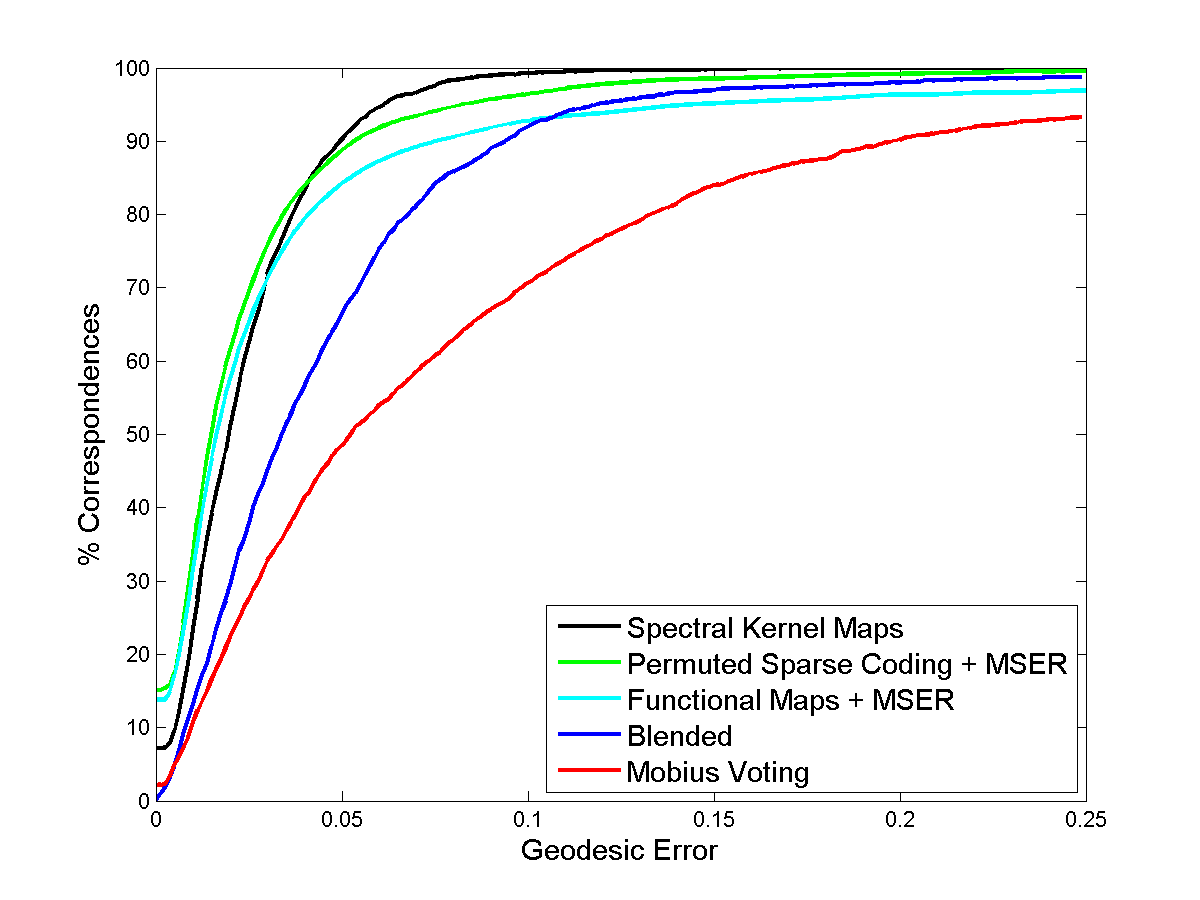}
	\end{overpic}
	\end{center}
	\caption{\small Evaluation of the spectral kernel maps algorithm applied to shapes from the SCAPE database, 
	             using the protocol of \cite{kim2011blended} with allowed symmetries.}
	\label{fig:SCAPE_Correspondences}
\end{figure}

\begin{table}[H]
\centering
\begin{tabular}{|c|c c c c c |} 
\hline 
Geodesic Error & 0.025  & 0.050  & 0.100  & 0.150  & 0.200 \\ 
\hline 
Spectral Kernel Maps&  77.9 &  94.4 &  99.3 &  99.9 & 100.0 \\
F. Maps + Blended&  69.5 &  88.7 &  96.4 &  98.5 &  99.2 \\
Blended&  55.9 &  84.7 &  96.6 &  98.0 &  98.8 \\
M{\"o}bius Voting&  39.3 &  60.9 &  79.6 &  86.2 &  89.4 \\ 
\hline 
\end{tabular}
\caption{Percentage of surface points falling within a relative geodesic error for different methods.} 
\label{tbl:ge}
\end{table}
Table \ref{tbl:ge} displays the percentage of correspondences that fall within different values of relative geodesic distances. It is interesting to focus on large geodesic errors.
Unlike other methods, in the proposed approach more than 99\% of the correspondences have significantly smaller geodesic error of less then than 0.1.

Next, we analyze the contribution of each component of the proposed framework.
Figure \ref{fig:TOSCA_MODULE_Correspondences} compares different combinations of the modules of the correspondence framework. We notice that the functional spectral kernel maps part is the most dominant module. Still, it needs a good starting point that is provided by the spectral quasi-conformal maps.

\begin{figure}[th]
	\center{{\bf Spectral kernel maps module analysis} \\ (TOSCA)}
	\begin{center}
	\begin{overpic}[width=1.0\columnwidth]{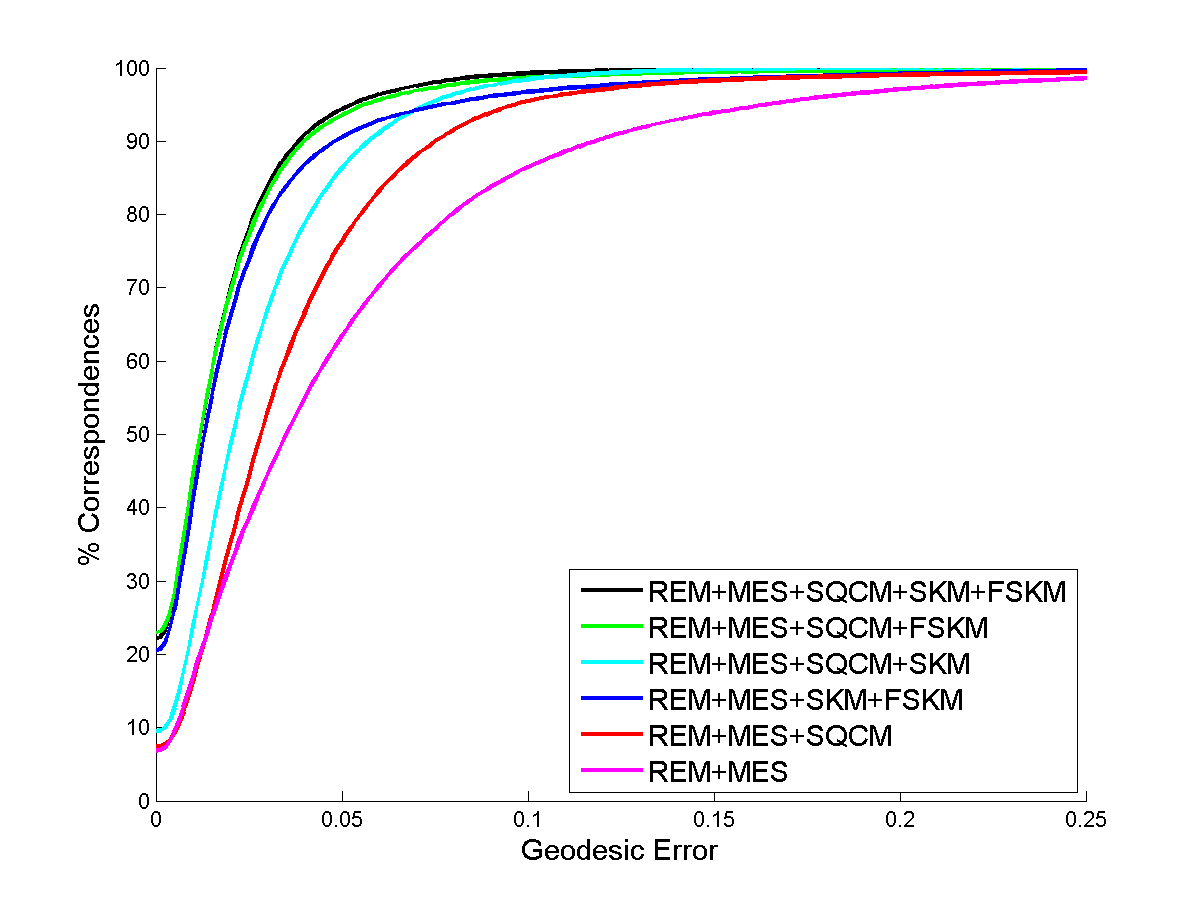}
	\end{overpic}
	\end{center}
	\caption{\small Evaluation of different combinations of the modules of the correspondence framework applied to shapes from the TOSCA database. The modules: Robust Eigenfunction Matching (REM), Matched Eigenbasis Signature (MES),
		Spectral Quasi-Conformal Maps (SQCM), Spectral Kernel Maps (SKM), Functional Spectral Kernel Maps (FSKM).}
	\label{fig:TOSCA_MODULE_Correspondences}
\end{figure}

We illustrate how our method is able to find the intrinsic reflective symmetry axis of nonrigid shapes.
Intrinsic symmetry detection can be viewed as finding correspondence from a shape to itself \cite{raviv2007symmetries} . Following this approach, we simply map a shape to itself and select the sign sequence that flips the orientation of the shape by maximizing $-Q(\varphi(\boldsymbol s))$.
In Figure \ref{fig:symmetry} we visualize the distance between a point and its image for several shapes from the TOSCA database.

\begin{figure}[th]
	\begin{center}
	\begin{overpic}[width=1.0\columnwidth]{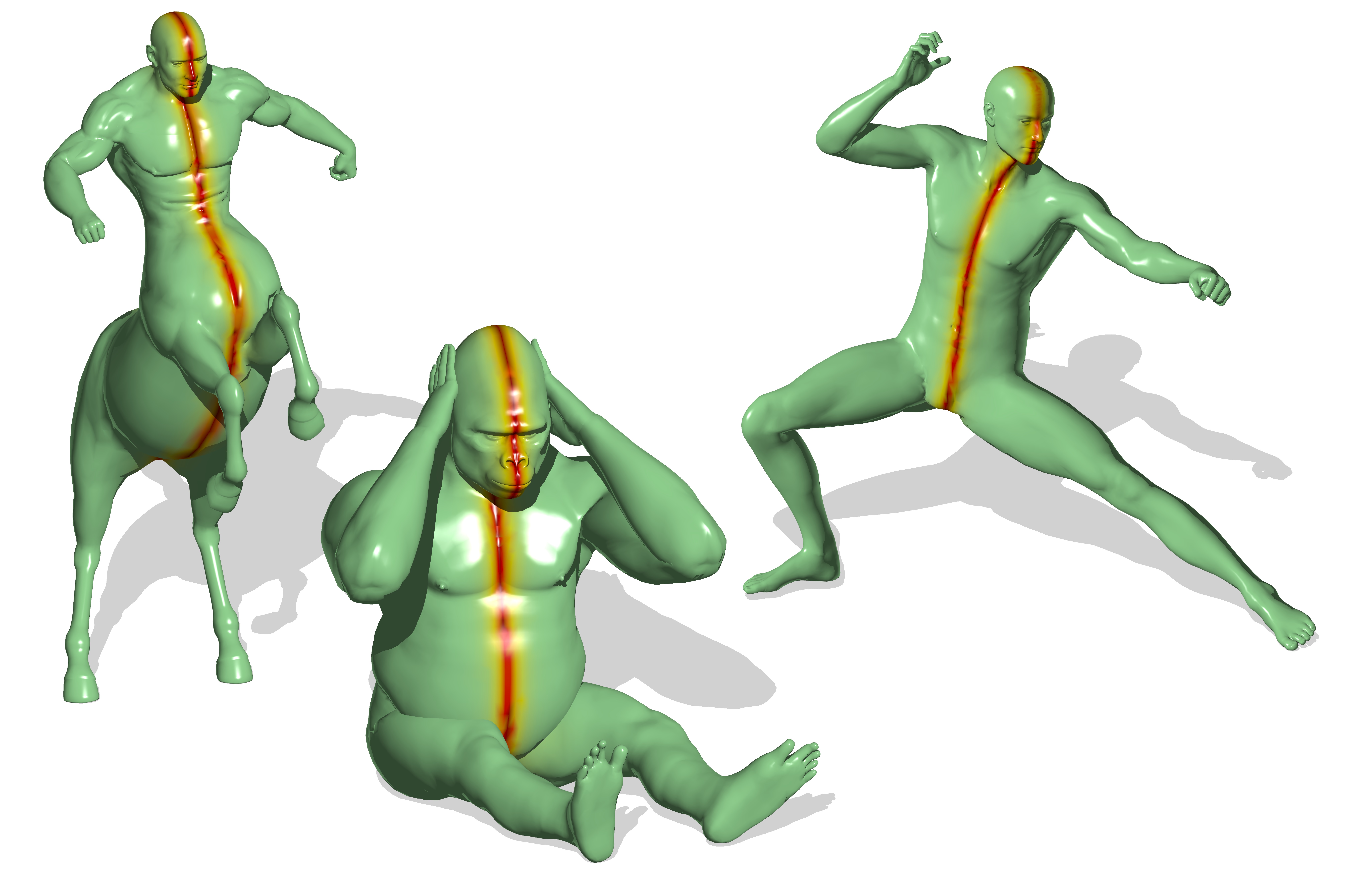}
	\end{overpic}
	\end{center}
	\caption{\small Symmetry axis of several shapes from the TOSCA database.}
	\label{fig:symmetry}
\end{figure}

Finally, we demonstrate the proposed approach in texture transfer experiments on shapes of scanned human faces taken from the BFM database \cite{paysan20093d}. The intrinsic correspondence gives us the possibility
to perform texture manipulation and mapping \cite{zigelman2002texture,Bronstein07calculusof}.
In Figure \ref{fig:texture_transfer}, we used the maps obtained by the proposed correspondence framework to transfer the textures of the reference shapes to the target shapes.
Figure \ref{fig:texture_morph} shows gradual morphing from one shape (left) to another shape (right), obtained
by linear interpolation of the extrinsic geometry and the texture, for different values of the interpolation factor $\gamma$ \cite{Bronstein07calculusof}.
Figure \ref{fig:face_math} depicts how the texture from the two middle human faces are transferred to the intrinsically symmetric halves of the target face.

\begin{figure*}[tp]
\begin{center}
\begin{overpic}[width=0.15\columnwidth]{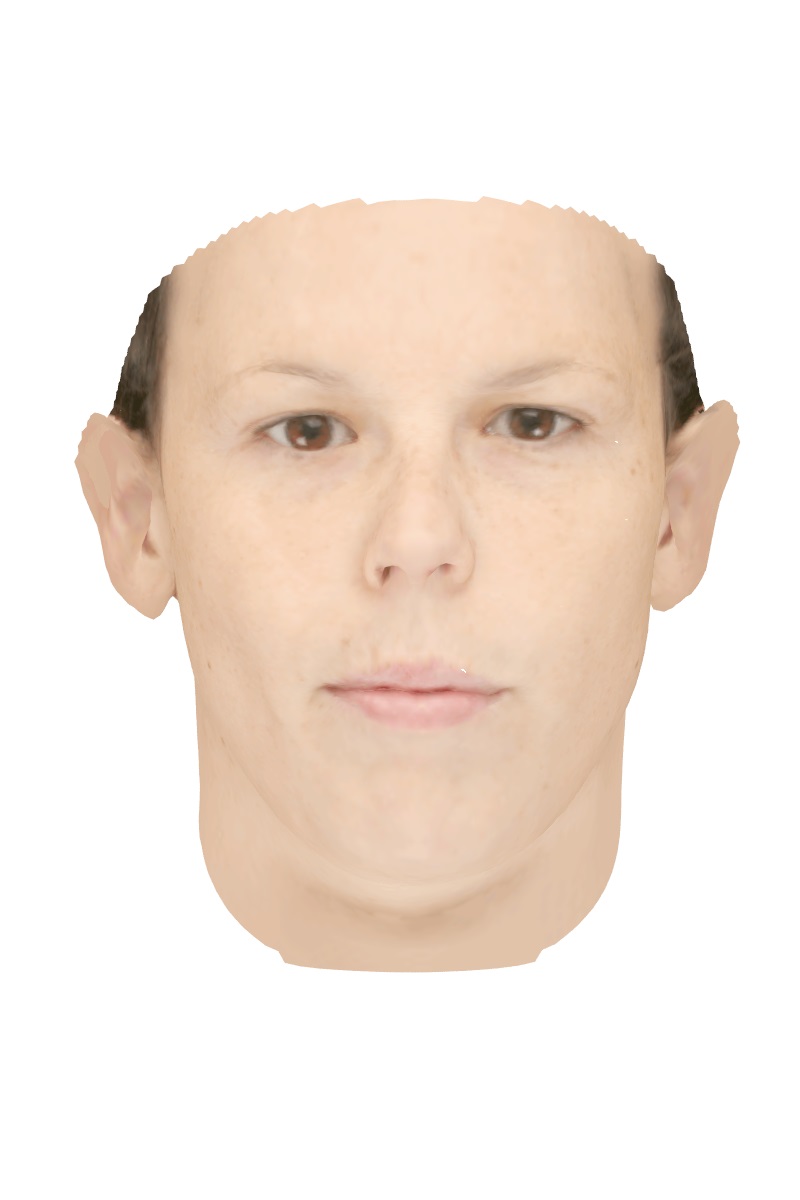}\end{overpic}
\hspace{-.25cm}
\begin{overpic}[width=0.15\columnwidth]{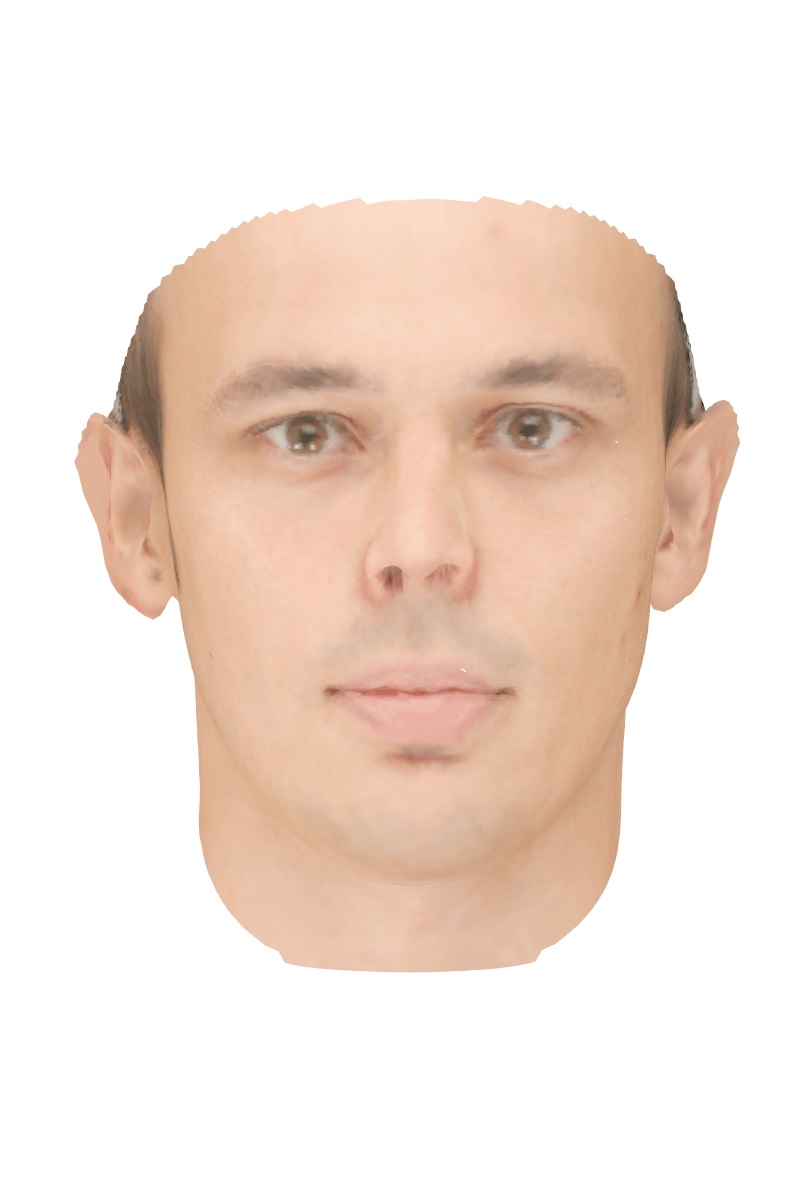}\end{overpic}
\hspace{-.25cm}
\begin{overpic}[width=0.15\columnwidth]{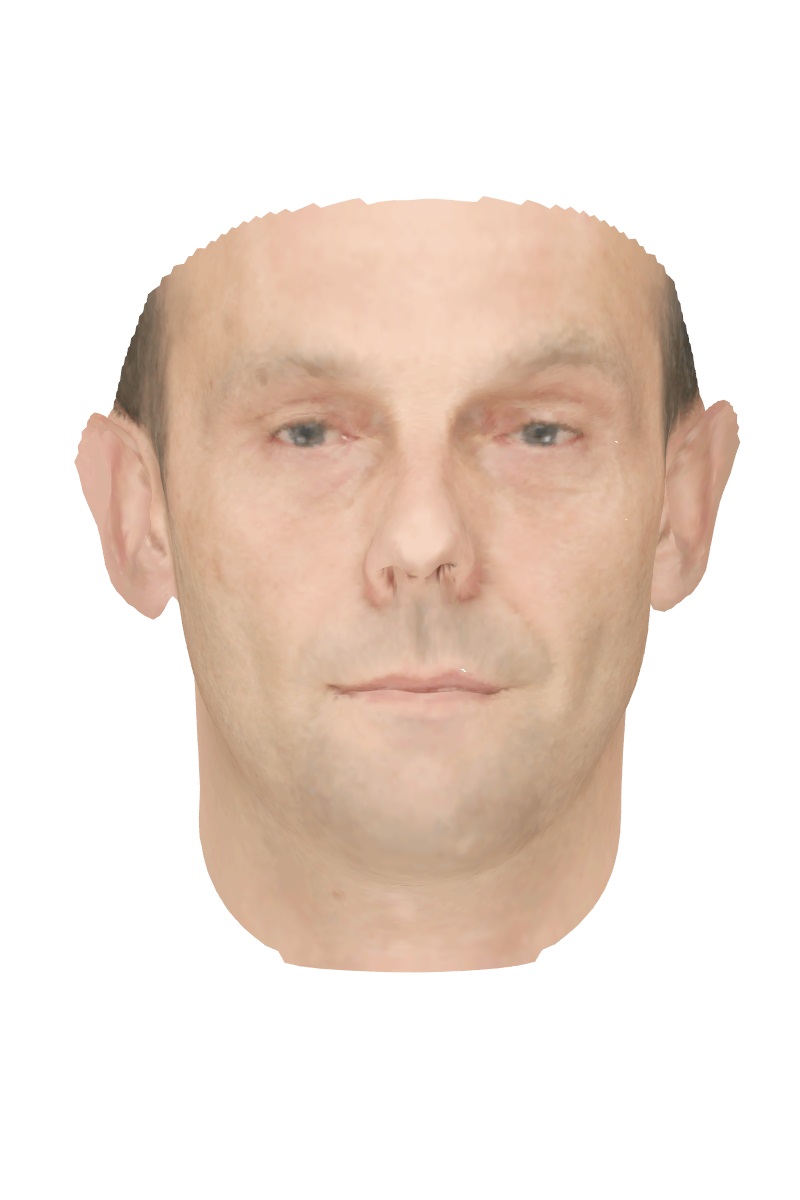}\end{overpic}
\hspace{0.5cm}
\begin{overpic}[width=0.15\columnwidth]{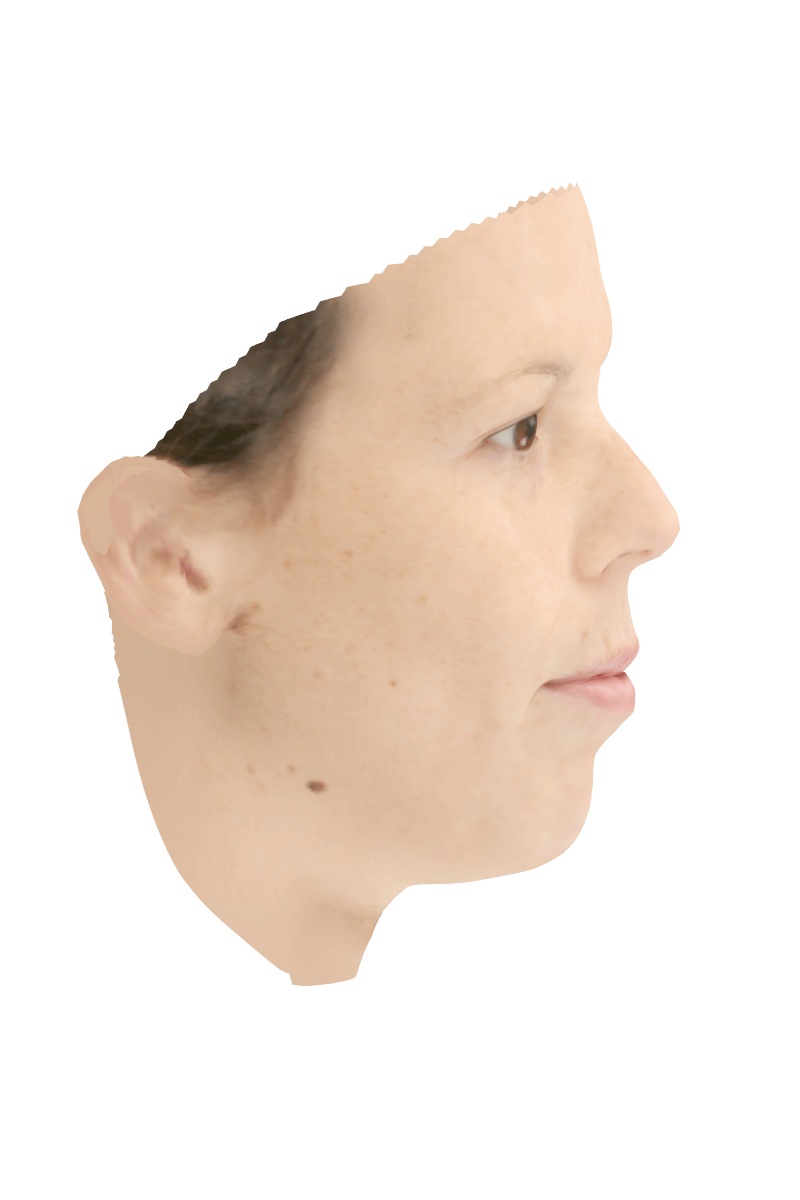}\end{overpic}
\hspace{-.25cm}
\begin{overpic}[width=0.15\columnwidth]{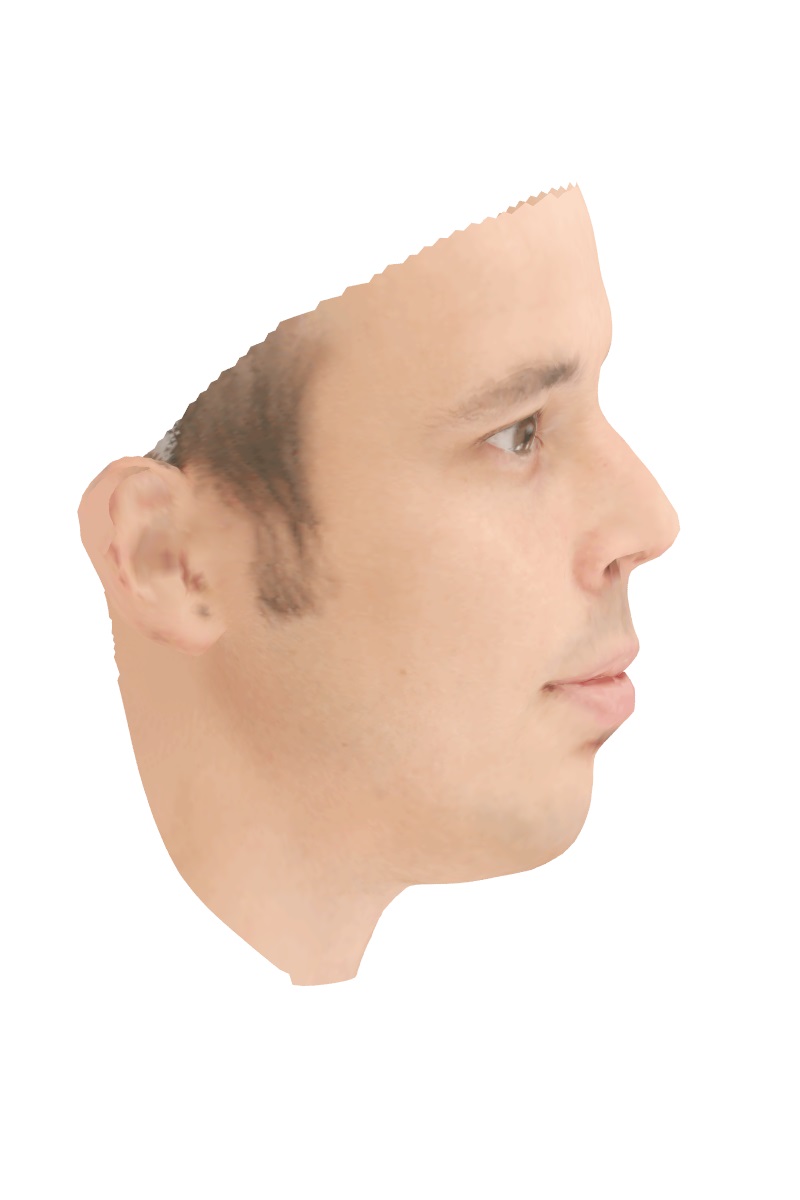}\end{overpic}
\hspace{-.25cm}
\begin{overpic}[width=0.15\columnwidth]{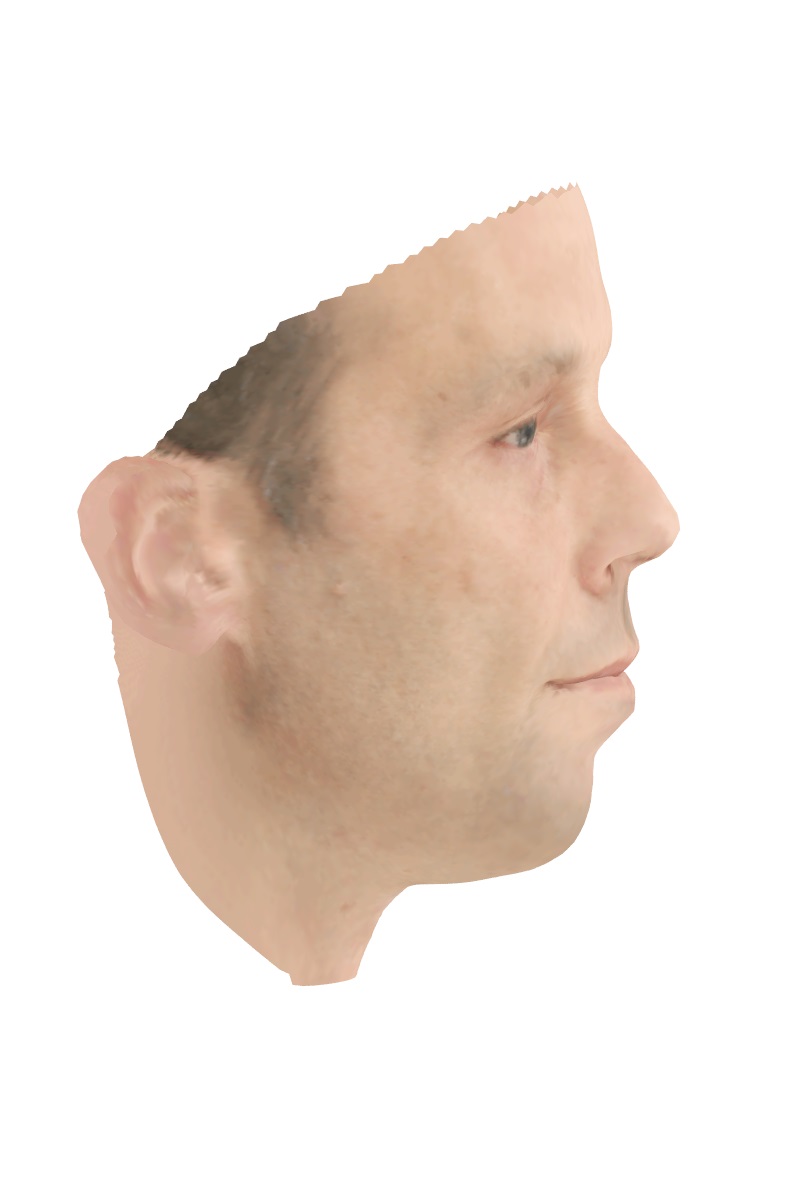}\end{overpic}
\end{center}
\vspace{-1.2cm}
\begin{center}
\begin{overpic}[width=0.15\columnwidth]{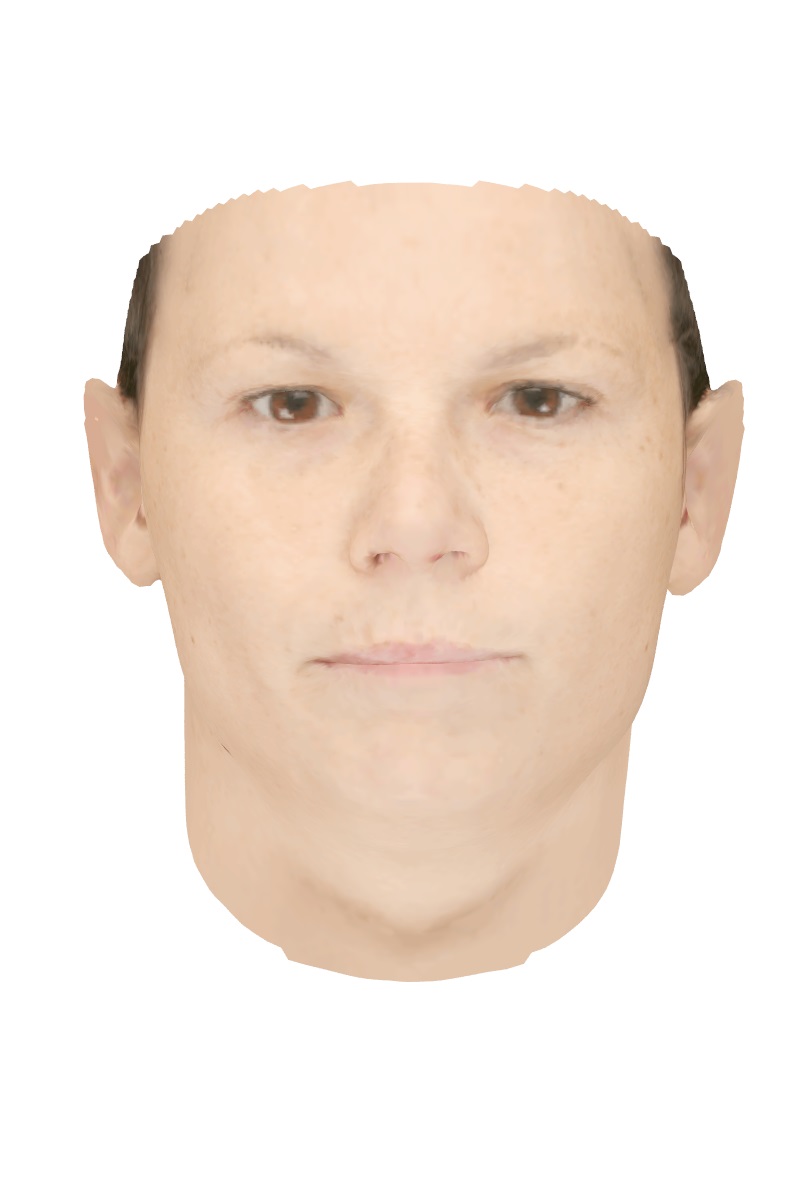}\end{overpic}
\hspace{-.25cm}
\begin{overpic}[width=0.15\columnwidth]{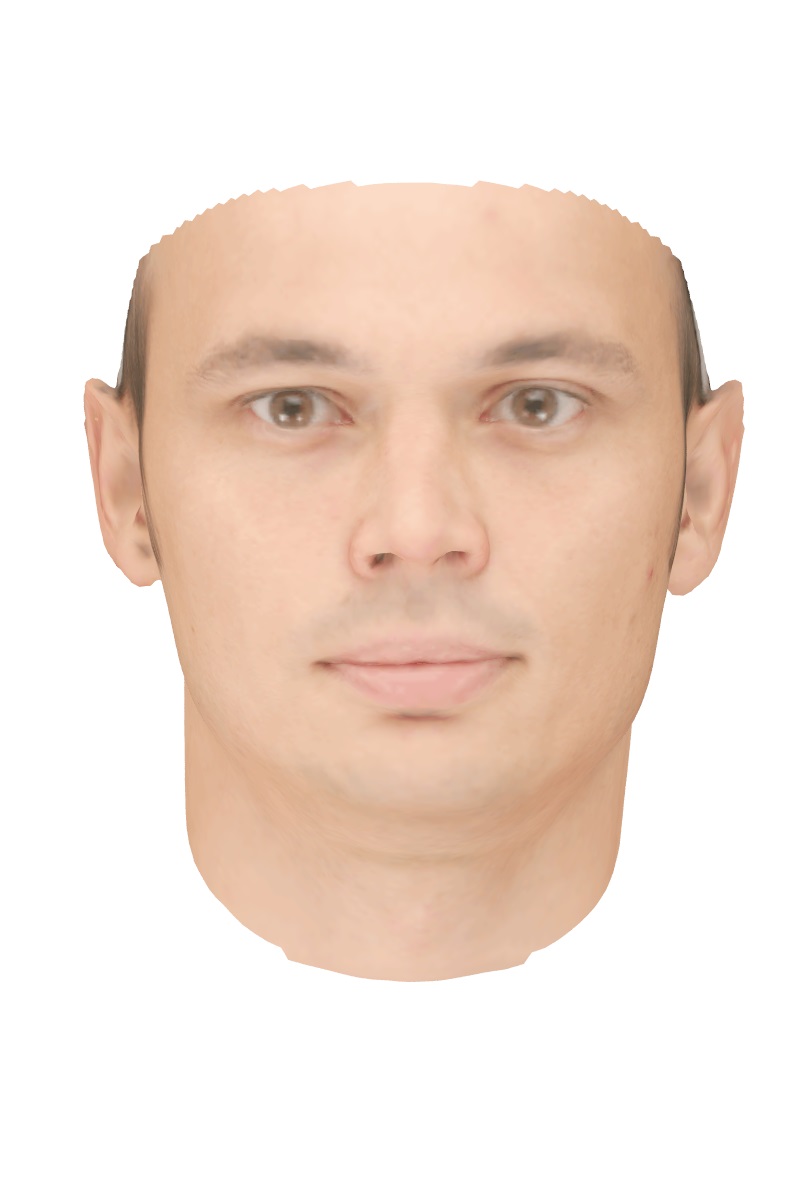}\end{overpic}
\hspace{-.25cm}
\begin{overpic}[width=0.15\columnwidth]{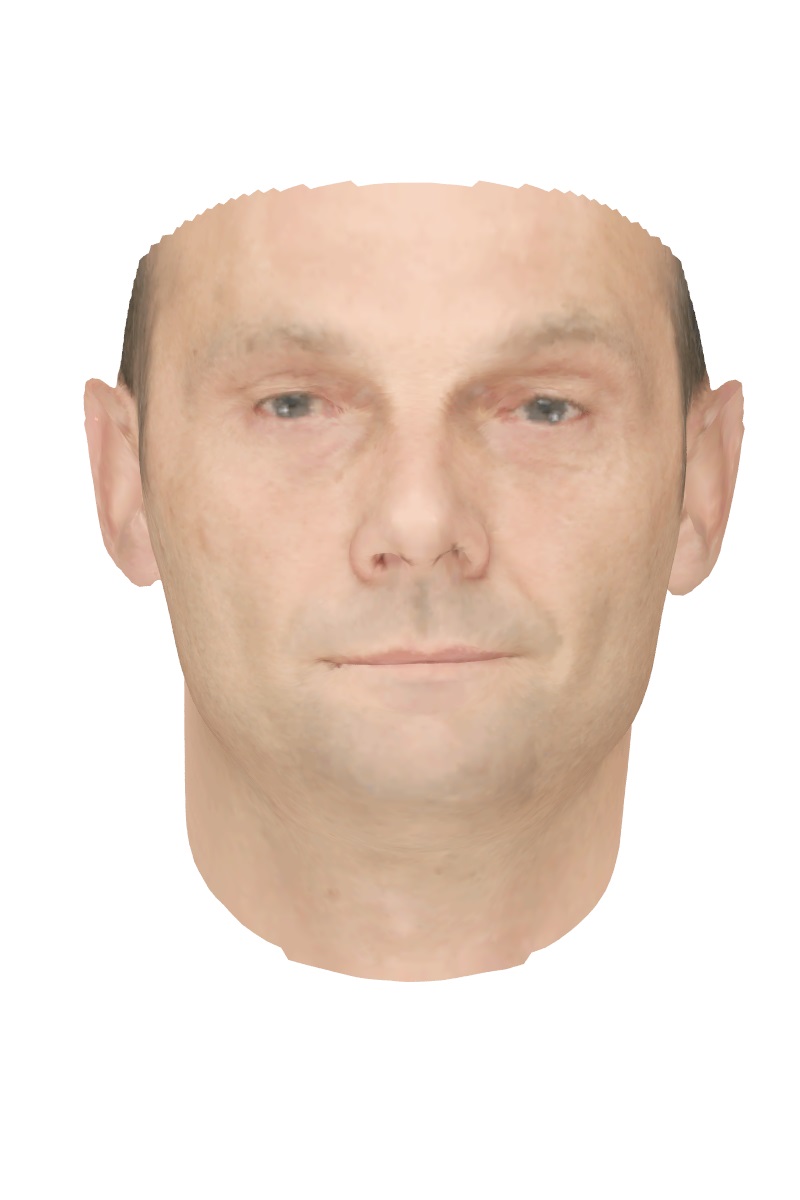}\end{overpic}
\hspace{0.5cm}
\begin{overpic}[width=0.15\columnwidth]{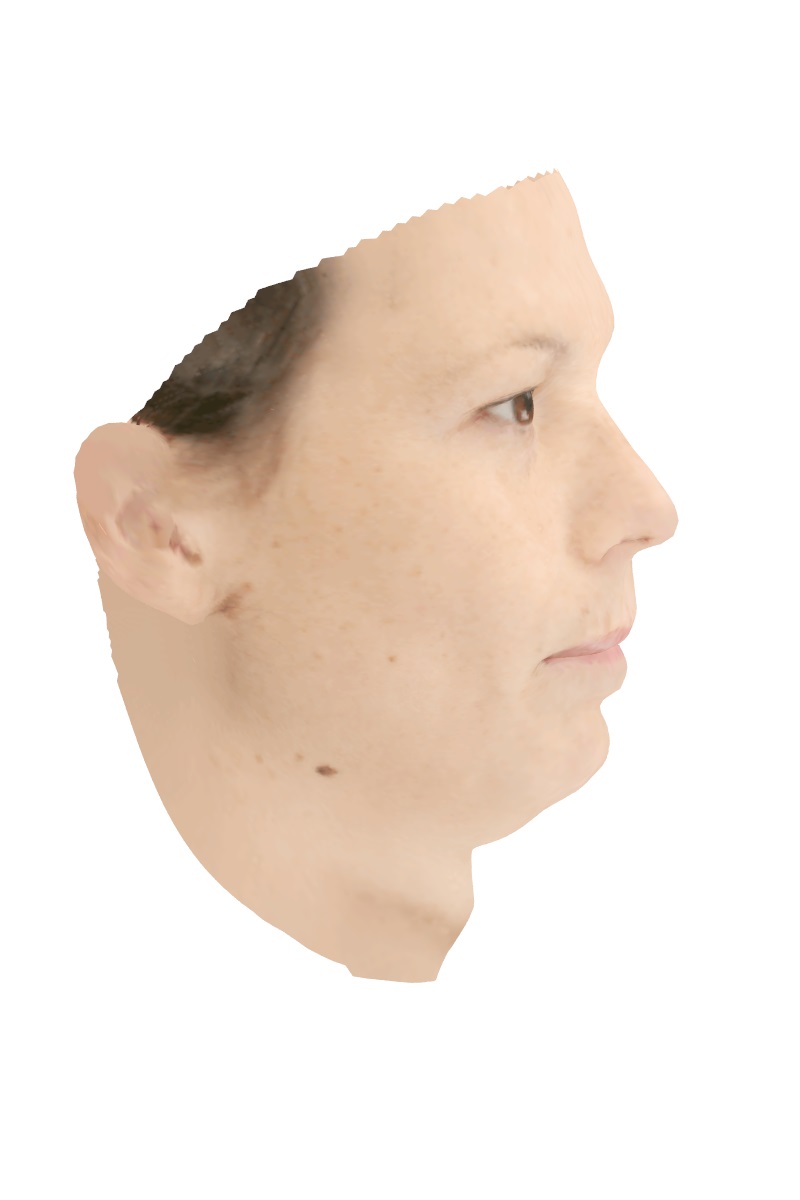}\end{overpic}
\hspace{-.25cm}
\begin{overpic}[width=0.15\columnwidth]{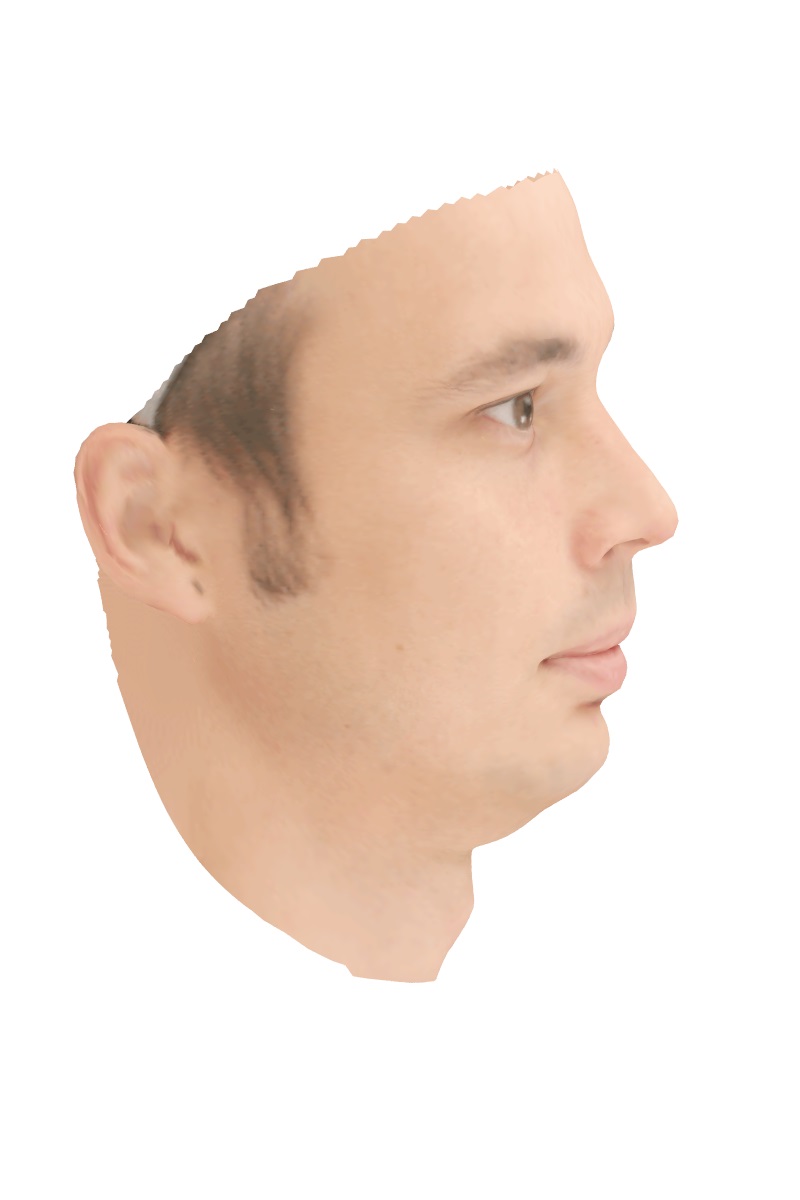}\end{overpic}
\hspace{-.25cm}
\begin{overpic}[width=0.15\columnwidth]{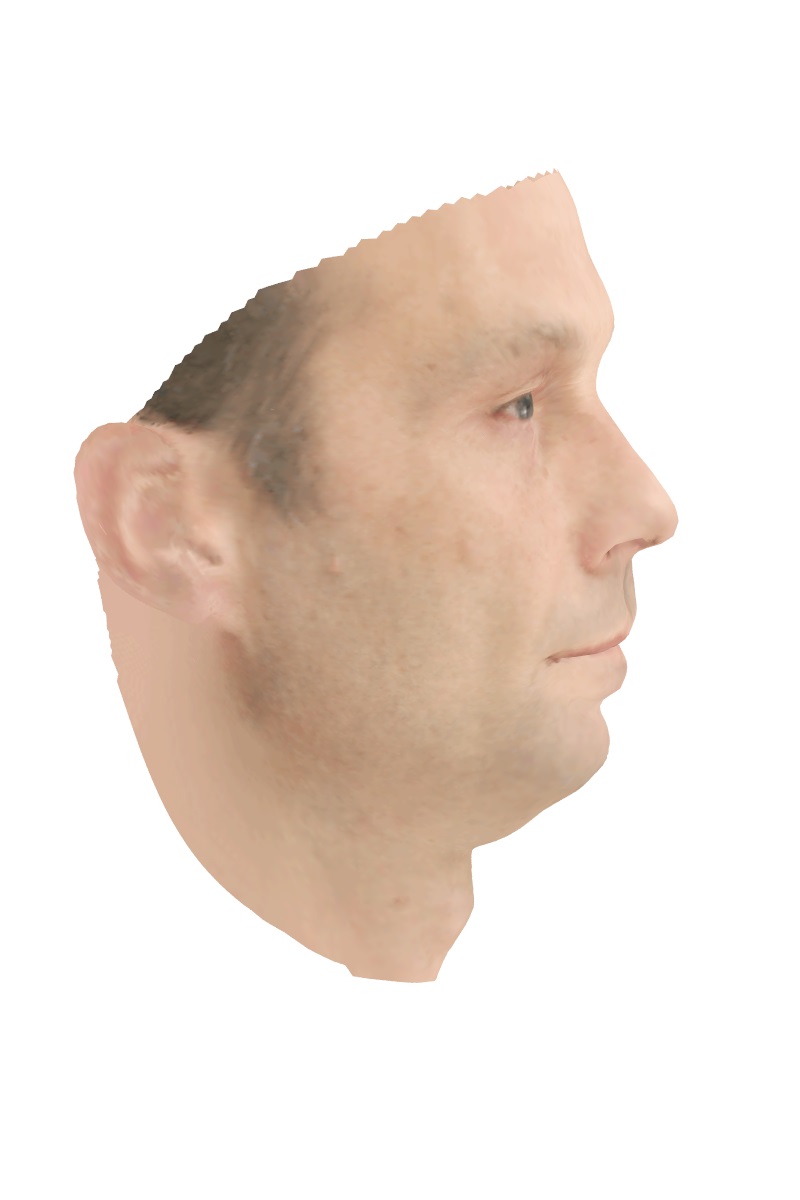}\end{overpic}
\end{center}
\vspace{-1.2cm}
\begin{center}
\begin{overpic}[width=0.15\columnwidth]{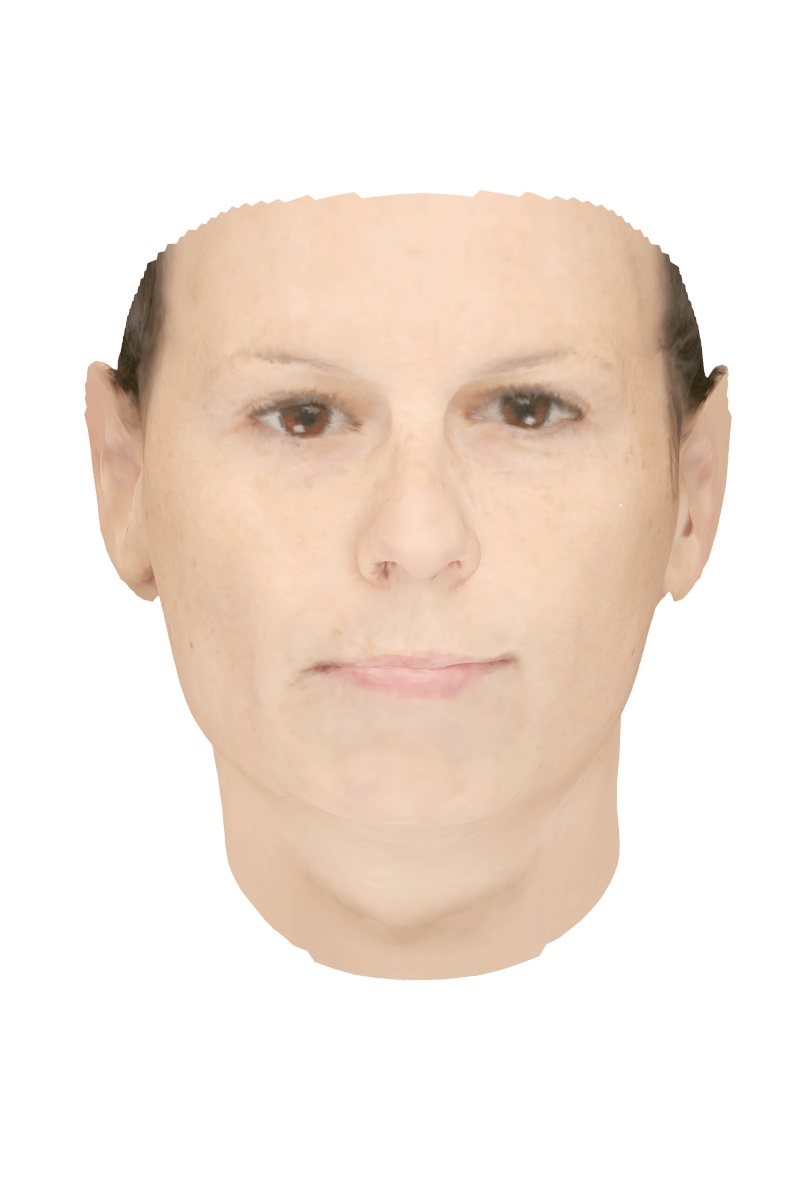}\end{overpic}
\hspace{-.25cm}
\begin{overpic}[width=0.15\columnwidth]{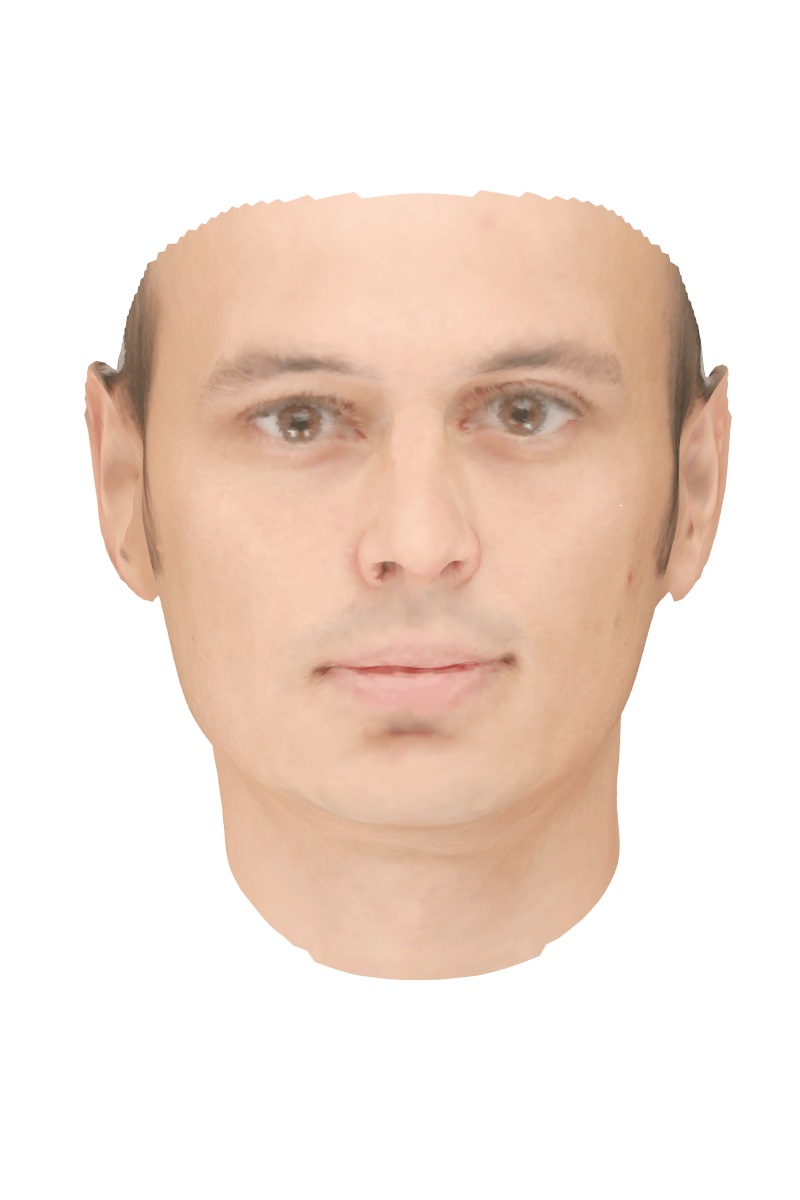}\end{overpic}
\hspace{-.25cm}
\begin{overpic}[width=0.15\columnwidth]{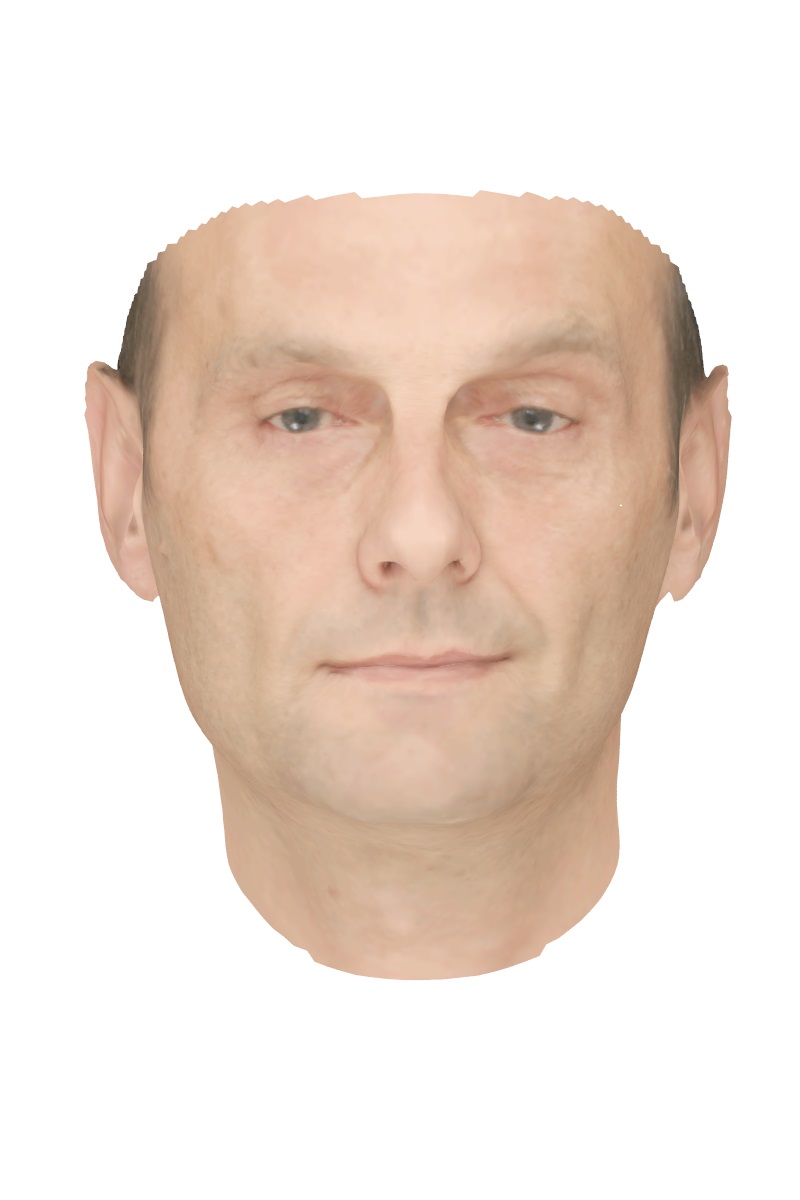}\end{overpic}
\hspace{0.5cm}
\begin{overpic}[width=0.15\columnwidth]{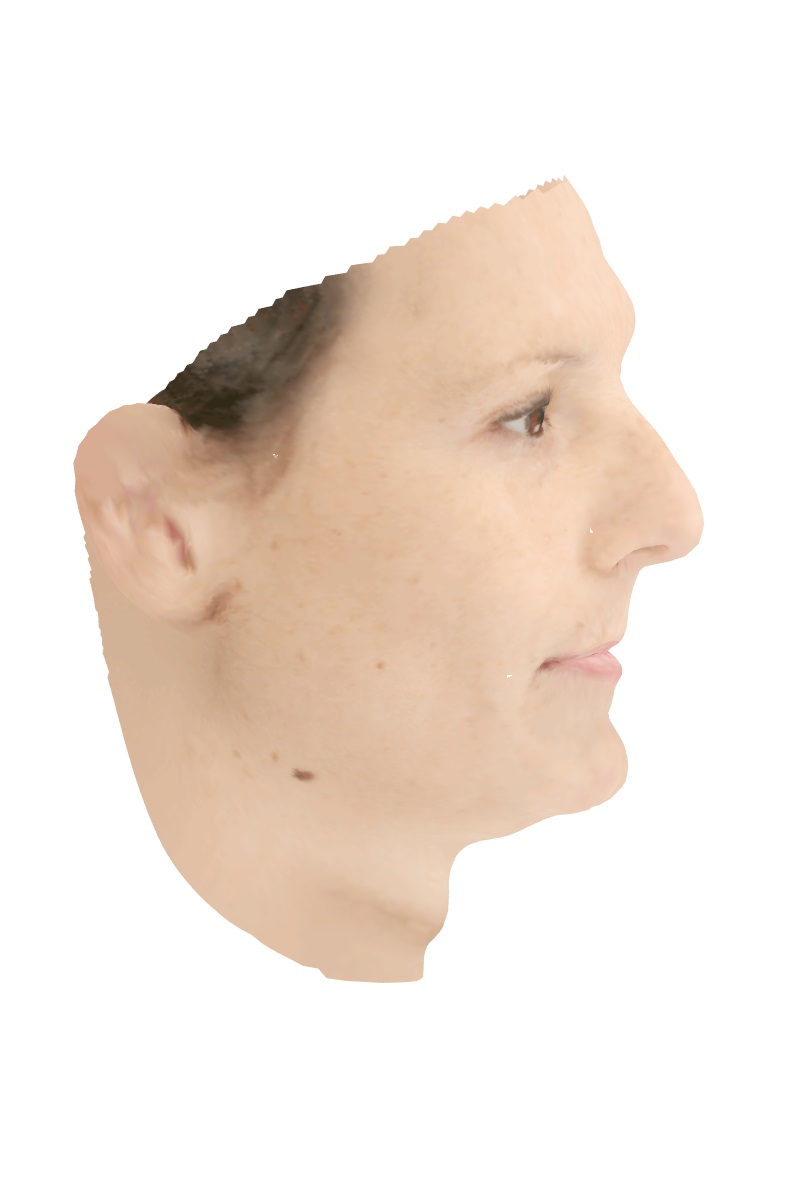}\end{overpic}
\hspace{-.25cm}
\begin{overpic}[width=0.15\columnwidth]{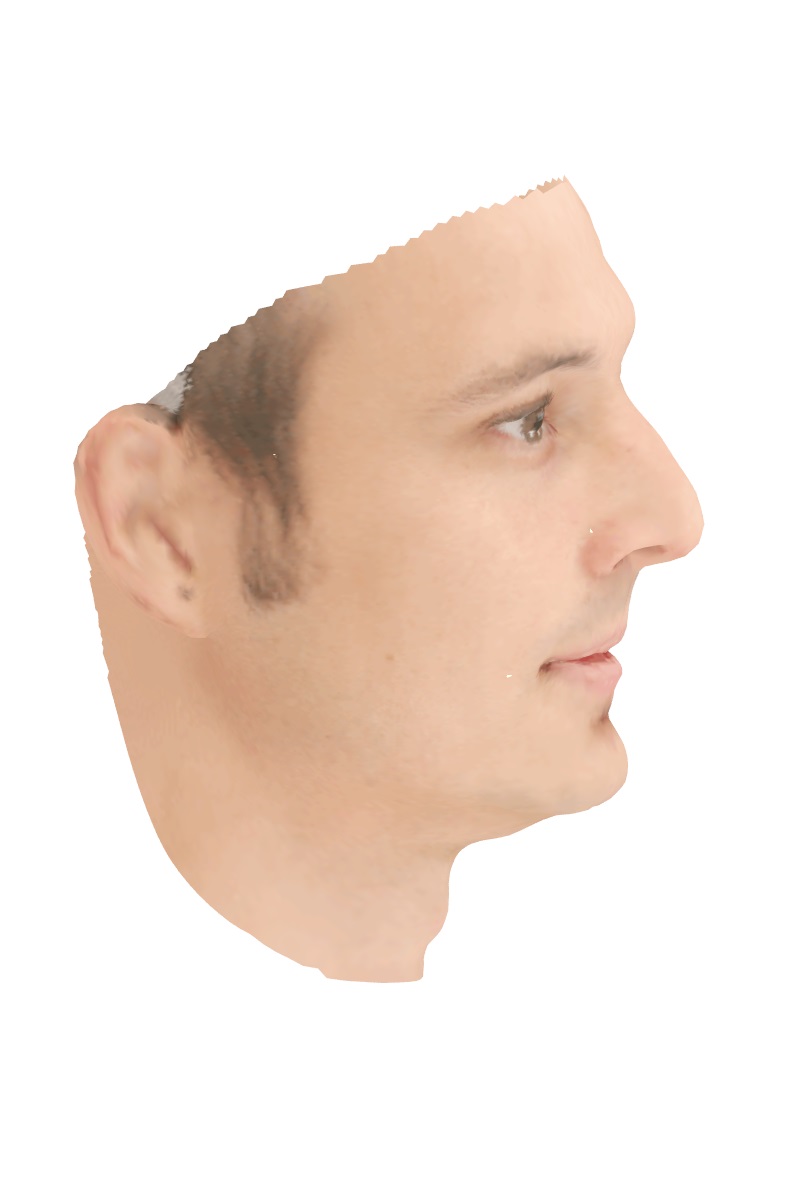}\end{overpic}
\hspace{-.25cm}
\begin{overpic}[width=0.15\columnwidth]{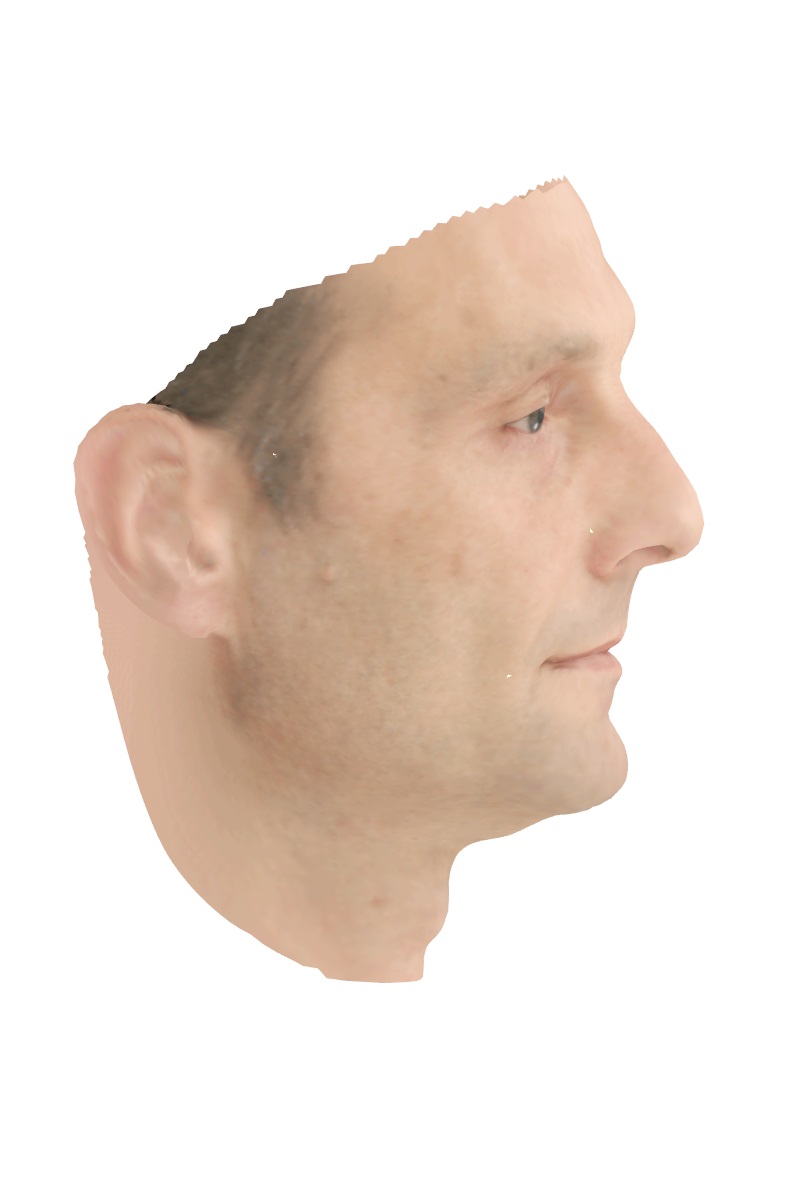}\end{overpic}
\hspace{-.25cm}
\end{center}
\vspace{-0.5cm}
	\caption{\small Texture transfer - front and side views. In each row, the same shape is shown with different textures transferred from reference faces (diagonal).}
	\label{fig:texture_transfer}
\end{figure*}

\begin{figure*}[tp]
\begin{center}
\begin{overpic}[width=0.15\columnwidth]{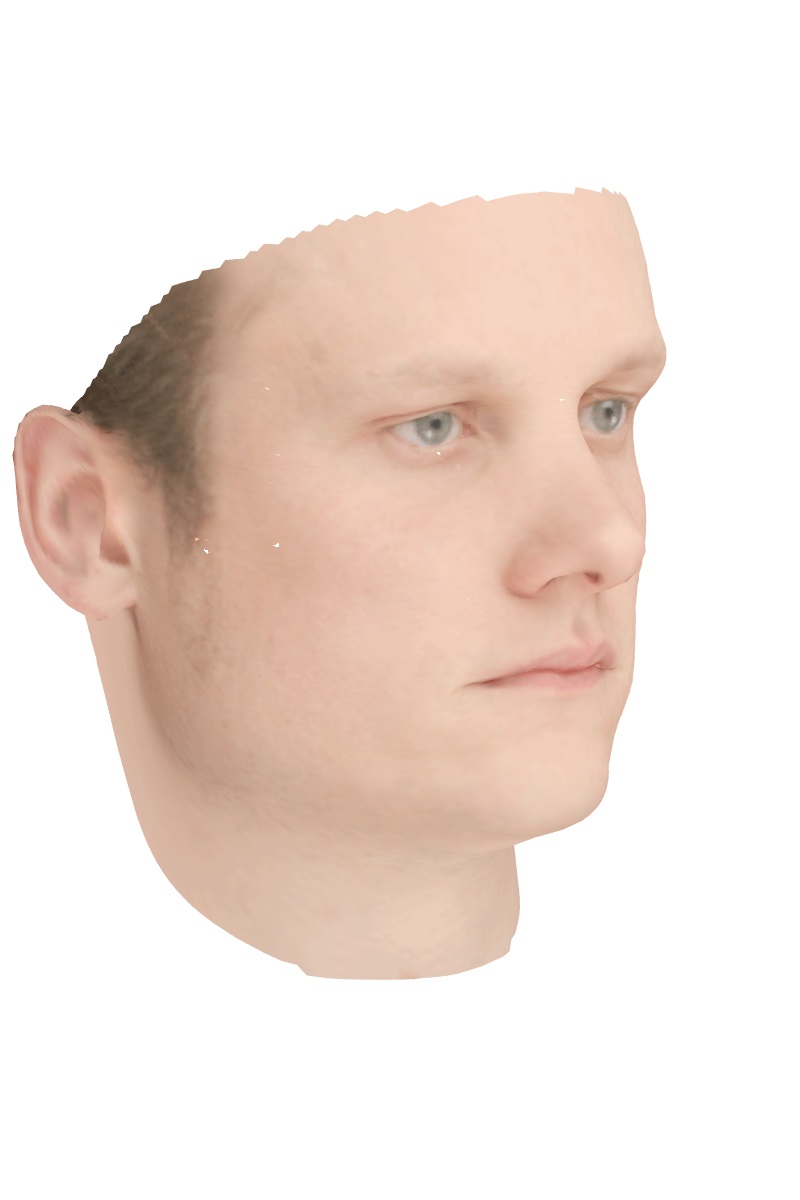}\put(22,7){$\gamma=0$}\end{overpic}
\hspace{-.6cm}
\begin{overpic}[width=0.15\columnwidth]{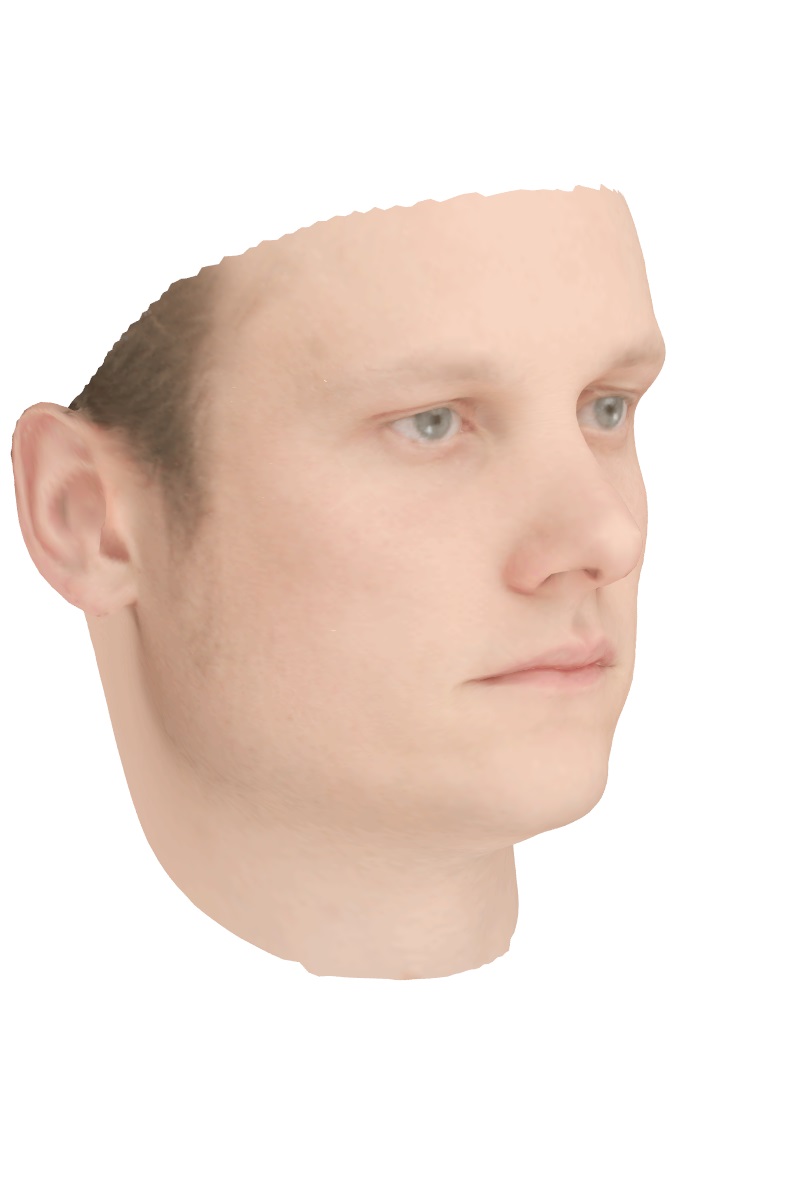}\put(19,7){$\gamma=0.1$}\end{overpic}
\hspace{-.6cm}
\begin{overpic}[width=0.15\columnwidth]{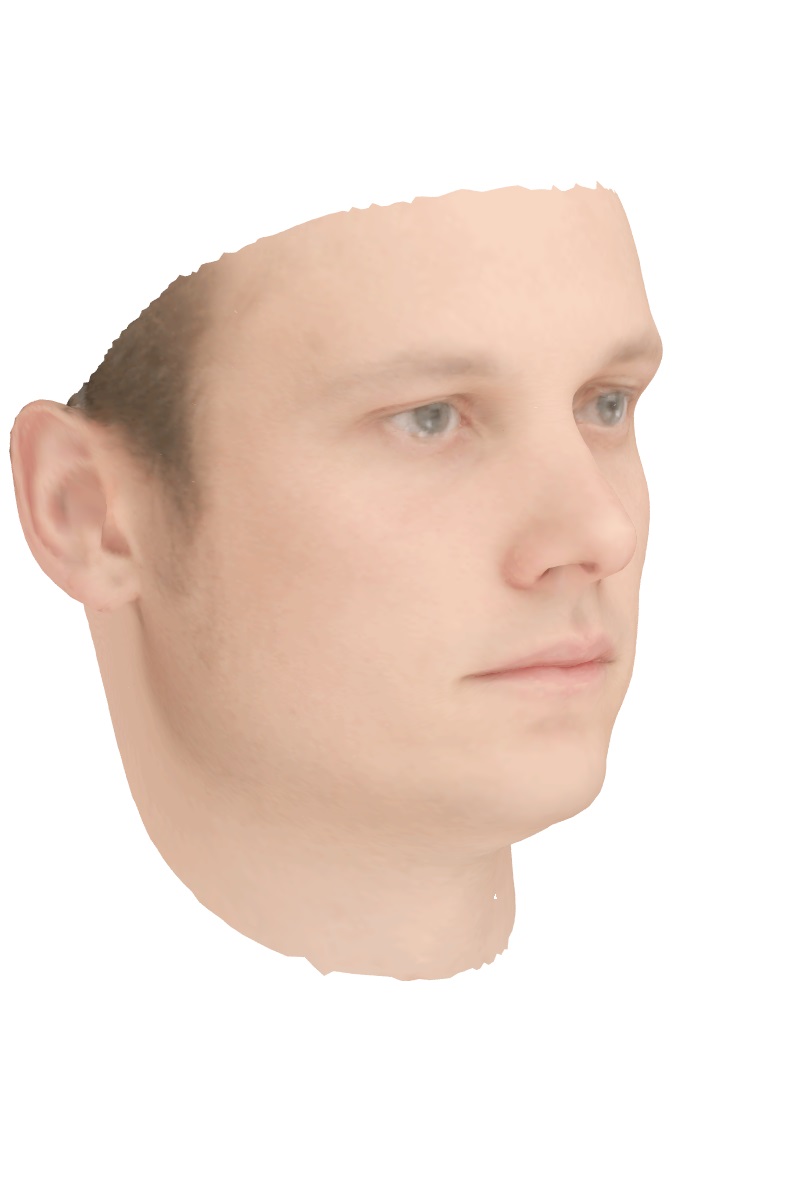}\put(19,7){$\gamma=0.3$}\end{overpic}
\hspace{-.6cm}
\begin{overpic}[width=0.15\columnwidth]{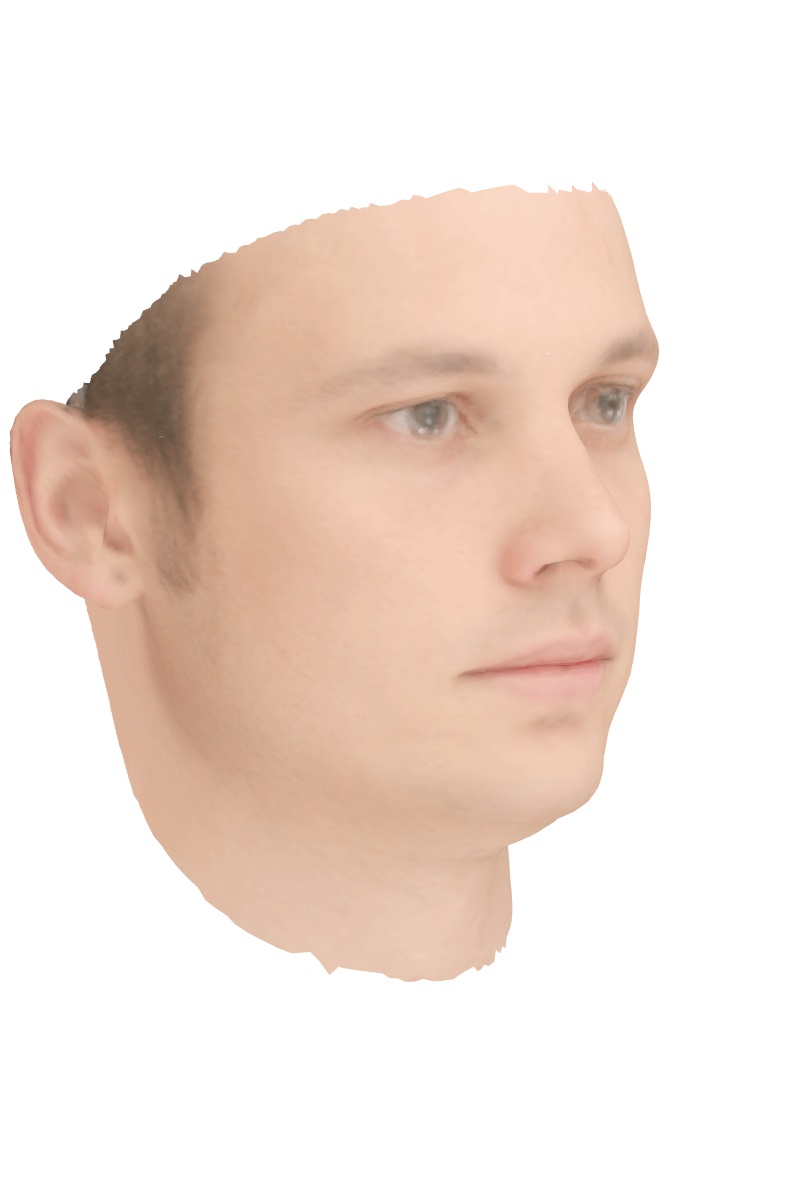}\put(19,7){$\gamma=0.5$}\end{overpic}
\hspace{-.6cm}
\begin{overpic}[width=0.15\columnwidth]{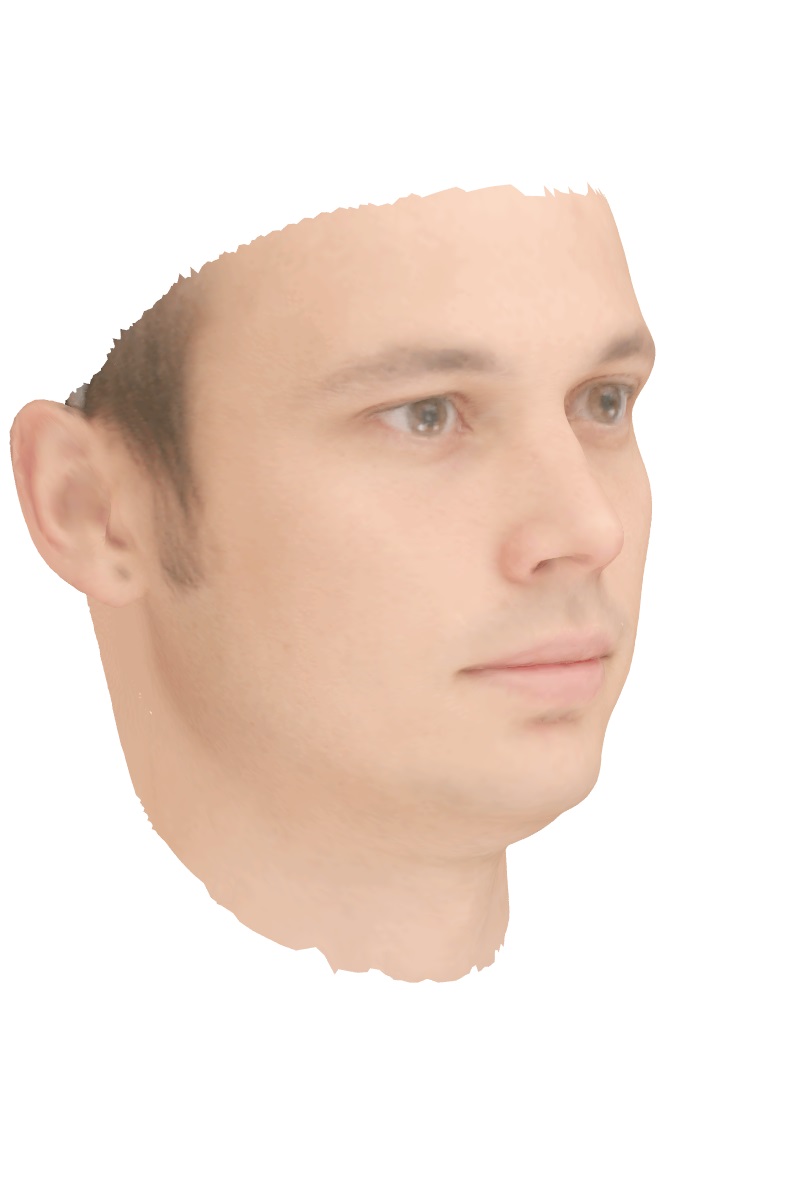}\put(19,7){$\gamma=0.7$}\end{overpic}
\hspace{-.6cm}
\begin{overpic}[width=0.15\columnwidth]{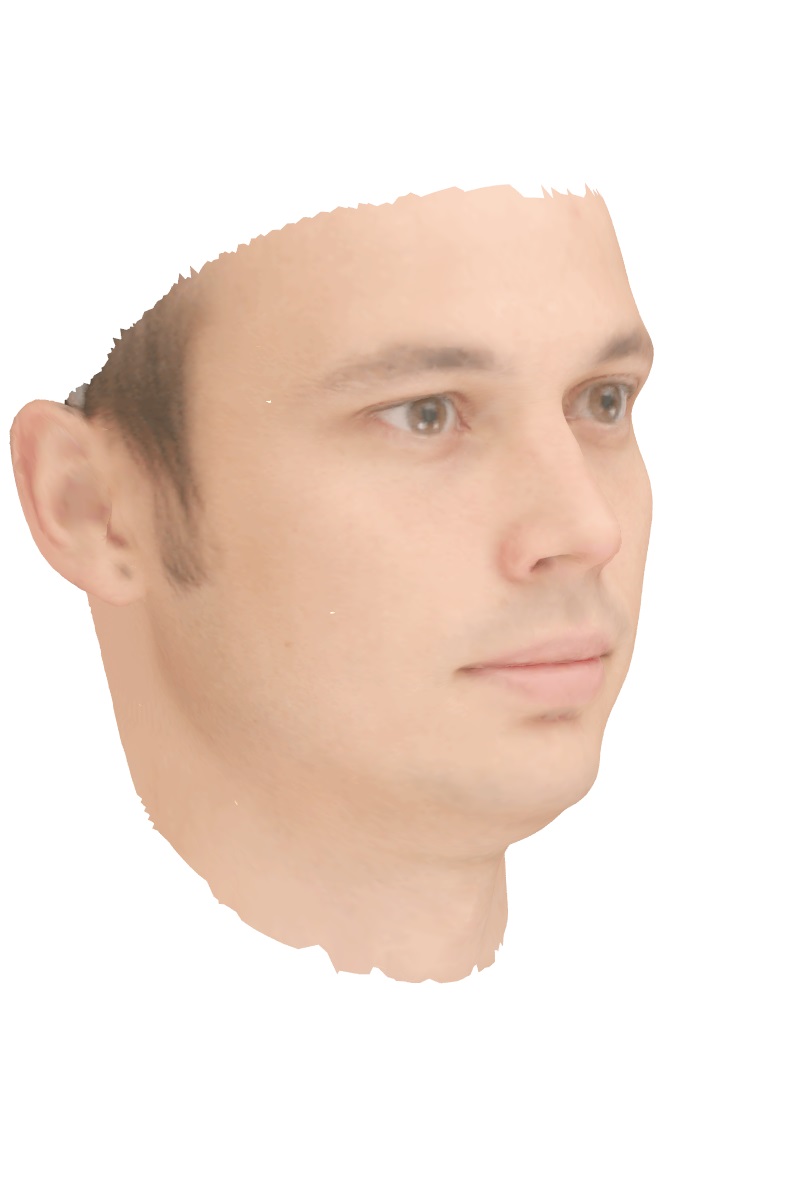}\put(19,7){$\gamma=0.8$}\end{overpic}
\hspace{-.6cm}
\begin{overpic}[width=0.15\columnwidth]{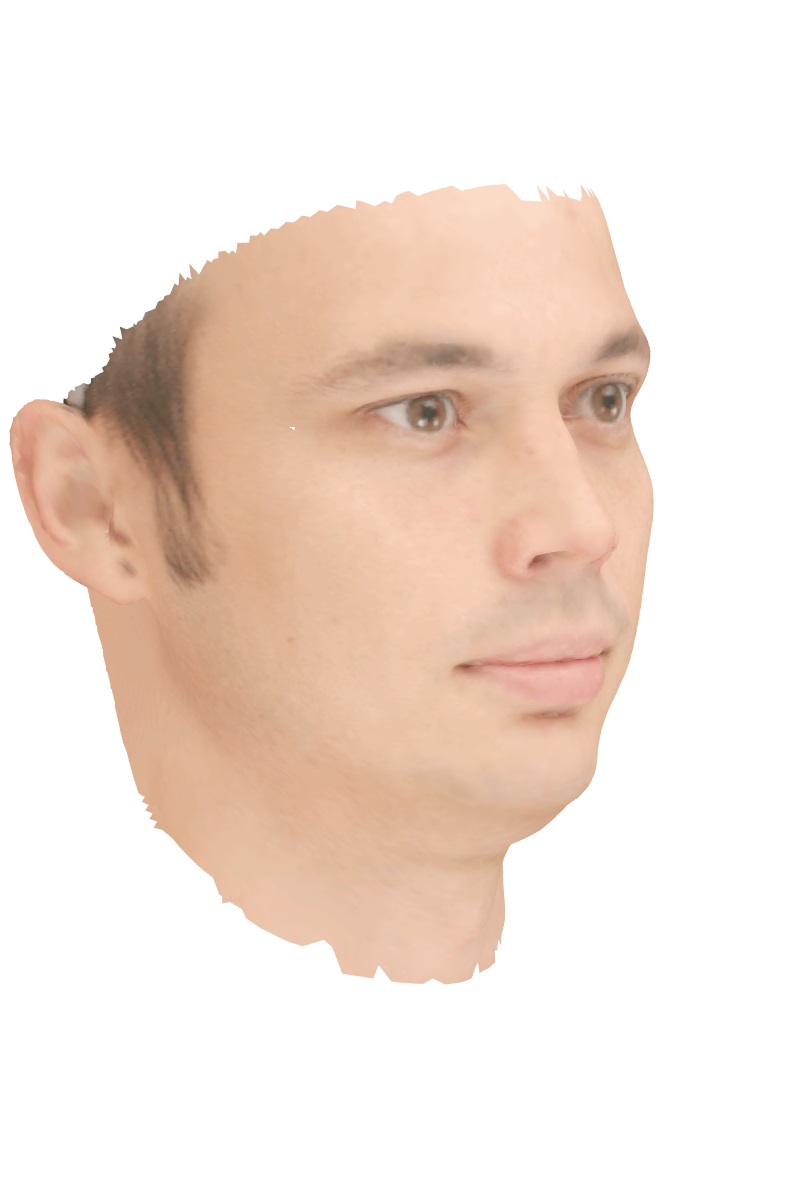}\put(22,7){$\gamma=1$}\end{overpic}
\end{center}
\vspace{-0.5cm}
	\caption{\small Shape morphing. The correspondence is used to transform the texture and the extrinsic geometry, creating a morphing effect. }
	\label{fig:texture_morph}
\end{figure*}

\section{Conclusions}
\label{sec:Conclusions}
A new method for comparing nonrigid shapes was introduced. The theoretical components are based on spectral distances between manifolds, that have been integrated into a holistic shape matching framework.
We have demonstrated the effectiveness of the proposed approach by achieving state-of-the-art
results on shape matching benchmarks.

In the future, we intend to study the characteristics of other types of distances between deformable shapes, and incorporate them in the proposed correspondence framework.

\ifdefined \TOG
\begin{acks}
\else
\subsection*{Acknowledgment}
\fi
This work has been supported by grant agreement no. 267414 of the European
Community's FP7-ERC program.
\ifdefined \TOG
\end{acks}
\fi

\appendix
\section{Proof of Theorem \ref{thm:pmet}}
\label{app:qc}
We would like to show that the quasi-conformal distance $d^{QC}_t:\mathcal{M} \times \mathcal{M} \mapsto \mathbb{R}_{\ge 0}$ satisfies the following properties.
\begin{enumerate}
	\item[P1:] Symmetry - $d^{QC}_t(X,Y) = d^{QC}_t(Y,X)$ for every $X$, $Y \in \mathcal{M}$.
	\item[P2:] Triangle-inequality - For every three Riemannian manifolds $X, Y, Z \in \mathcal{M}$,
	\begin{eqnarray*}
			d^{QC}_t(X,Z) &\le& d^{QC}_t(X,Y) + d^{QC}_t(Y,Z).
	\end{eqnarray*}
	\item[P3:] Identity - 
\begin{enumerate}
		\item If $X,Y \in \mathcal{M}$ are isometric Riemannian manifolds, then $d^{QC}_t(X,Y) = 0$. 
		\item If $d^{QC}_t(X,Y) = 0$ then $X$, $Y$ are isometric.
\end{enumerate}
\end{enumerate}
The proof of the theorem is based on the work of B{\'e}rard \etal \cite{berard1994embedding}, that showed that the spectral embedding distance is a metric.
\begin{lemma}
\label{lem:jac}
		Let $X,Y \in \mathcal{M}$, and let $\varphi : X \mapsto Y$ and $\psi : Y \mapsto X$ be smooth bijective mappings $y=\varphi(\psi(y))$.
		Let us denote by $\mathcal{J}_{\varphi}(x)$, $\mathcal{J}_{\psi}(y)$, the Jacobians of the maps $\varphi$, $\psi$, respectively. 
		Then, $d^{\text{QC}}(X, Y)=0$ implies that $\mathcal{J}_{\varphi}(x)$  (and equivalently $\mathcal{J}_{\psi}(y)$) is constant.
\end{lemma}
\begin{proof}
		By the definition of the spectral quasi-conformal distance, if $d^{\text{QC}}(X, Y)=0$ and if $x = \psi(y)$ or $y = \varphi(x)$, then 
		\begin{eqnarray}
\label{eq:qc_constraint0}
			\sqrt{\text{Vol}(X)}e^{-\lambda^X_i t/2}\phi^X_i &=& \sqrt{\text{Vol}(Y)}e^{-\lambda^Y_i t/2}\phi^Y_i  \,\: \forall \, i>0.
		\end{eqnarray}
		Integrating Eq. \ref{eq:qc_constraint0}, we get
		\begin{eqnarray}
			0 &=& \sqrt{\text{Vol}(X)}e^{-\lambda^X_i t/2} \int_X \phi^X_i(x)  d\text{V}_X
			\cr &=& \sqrt{\text{Vol}(Y)}e^{-\lambda^Y_i t/2} \int_X \phi^Y_i(\varphi(x))  d\text{V}_X
			\cr &=& \sqrt{\text{Vol}(Y)}e^{-\lambda^Y_i t/2} \int_Y \phi^Y_i(y)  \mathcal{J}_{\psi}(y) d\text{V}_Y.
		\end{eqnarray}
		Hence, $\mathcal{J}_{\psi}(y)$ is orthogonal to $\phi^Y_i(y) \,\: \forall \, i>0$. 
		We conclude that $\mathcal{J}_{\psi}(y)$ (and equivalently $\mathcal{J}_{\varphi}(x)$)  must be a constant.
\end{proof}
\ifdefined \TOG
\begin{proof}[ 2]
\else
\begin{proof}[Proof 2]
\fi
		Under the condition that $\{\:\,{(\nabla^Y\phi^Y_i \cdot \nabla^Y\phi^Y_j)}_{\begin{smallmatrix}  i \ne j\end{smallmatrix}} \cup \text{ constant}\}$ spans the space of real scalar functions, we can prove the claim only by using the coordinates of $\tilde{I}^{\Phi}_t$.

		By the definition of the spectral quasi-conformal distance, if $d^{\text{QC}}(X, Y)=0$ and if $x = \psi(y)$ or $y = \varphi(x)$, then 
		\begin{eqnarray}
\label{eq:qc_constraint}
			z^X_{i,j}(\nabla^X\phi^X_i \cdot \nabla^X\phi^X_j )&=& z^Y_{i,j}(\nabla^Y\phi^Y_i \cdot \nabla^Y\phi^Y_j),
		\end{eqnarray}
where
		\begin{eqnarray*}
			z^X_{i,j} \: \equiv \: \text{Vol}(X)\cfrac{e^{-(\lambda^X_i+\lambda^X_j) t/2}}{\sqrt{\lambda^X_i\lambda^X_j}}, \quad
			 z^Y_{i,j} \: \equiv \: \text{Vol}(Y)\cfrac{e^{-(\lambda^Y_i+\lambda^Y_j) t/2}}{\sqrt{\lambda^Y_i\lambda^Y_j}}. \nonumber
		\end{eqnarray*}
		Integrating Eq. \ref{eq:qc_constraint}, we get
		\begin{eqnarray}
\label{eq:qc_int}
			\eqmark z^X_{i,j}\int_X \nabla^X\phi^X_i(x) \cdot \nabla^X\phi^X_j(x) d\text{V}_X \eqcr
			 &=& z^Y_{i,j}\int_X (\nabla^Y\phi^Y_i(\varphi(x)) \cdot \nabla^Y\phi^Y_j(\varphi(x))) d\text{V}_X
			\cr &=& z^Y_{i,j}\int_Y (\nabla^Y\phi^Y_i(y) \cdot \nabla^Y\phi^Y_j(y)) \mathcal{J}_{\psi}(y) d\text{V}_Y.
		\end{eqnarray}
		The gradients of the eigenfunctions are orthogonal
		\begin{eqnarray}
\label{eq:qc_green1}
		\int_X (\nabla^X\phi^X_i(x) \cdot \nabla^X\phi^X_j(x)) d\text{V}_X &=& 0 \,\: \forall \, i\ne j, \,\, 
		\cr \int_Y (\nabla^Y\phi^Y_i(y) \cdot \nabla^Y\phi^Y_j(y)) d\text{V}_Y &=& 0 \,\: \forall \, i\ne j.
		\end{eqnarray}
		and using Eq. \ref{eq:qc_int} and Eq.\ref{eq:qc_green1}, we have
		\begin{eqnarray*}
			\int_Y (\nabla^Y\phi^Y_i(y) \cdot \nabla^Y\phi^Y_j(y)) \mathcal{J}_{\psi}(y) d\text{V}_Y &=&  0.  \nonumber
		\end{eqnarray*}
		Hence, $\mathcal{J}_{\psi}(y)$ is orthogonal to $(\nabla^Y\phi^Y_i(y) \cdot \nabla^Y\phi^Y_j(y)) \,\: \forall \, i \ne j$.
		Again, we conclude that $\mathcal{J}_{\psi}(y)$ (and equivalently $\mathcal{J}_{\varphi}(x)$) is a constant.
\end{proof}
		Because the Jacobian  $\mathcal{J}_{\psi}(y)= \mathcal{J}_{\psi}$  is a constant, the following relation holds
		\begin{eqnarray*}
		 \text{Vol}(X) &=& \int_X d\text{V}_X \eqs = \eqs \int_Y \mathcal{J}_{\psi} d\text{V}_Y \eqs = \eqs \mathcal{J}_{\psi} \text{Vol}(Y). \nonumber
		\end{eqnarray*}
		Therefore, the Jacobian can be expressed as the ratio of the volumes of the Riemannian manifolds $\mathcal{J}_{\psi} = \text{Vol}(X)/ \text{Vol}(Y)$.

\begin{lemma}
\label{lem:isospec}
	For $X,Y \in \mathcal{M}$, $d^{\text{QC}}(X, Y)=0$ implies that $X$, $Y$ are isospectral, \ie  $\lambda^X_i = \lambda^Y_i \,\: \forall \, i$.
\end{lemma}
\begin{proof}
		Integrating the square of Eq. \ref{eq:qc_constraint0}, we get
		\begin{eqnarray*}
		\eqmark \int_X \big(\sqrt{\text{Vol}(X)}e^{-\lambda^X_i t/2}\phi^X_i(x)\big)^2 d\text{V}_X \eqcr
		&=& \int_Y \big(\sqrt{\text{Vol}(Y)}e^{-\lambda^Y_i t/2}\phi^Y_i(y)\big)^2 \mathcal{J}_{\psi}(y) d\text{V}_Y.
		\end{eqnarray*}
		Simplifying, using Lemma \ref{lem:jac}
		\begin{eqnarray*}
		\cr {\text{Vol}(X)}e^{-\lambda^X_i t} &=& {\text{Vol}(Y)}e^{-\lambda^Y_i t} \mathcal{J}_{\psi},
		\cr e^{-(\lambda^X_i t)} &=& e^{-(\lambda^Y_i t)},
		\cr \lambda^X_i &=&  \lambda^Y_i. \nonumber
		\end{eqnarray*}
\end{proof}
\ifdefined \TOG
\begin{proof}[ 2]
\else
\begin{proof}[Proof 2]
\fi
		For all eigenfunctions and the respective eigenvalues
		\begin{eqnarray}
\label{eq:qc_green2}
		\int_X \norm{\nabla^X\phi^X_i(x)}^2 d\text{V}_X &=& \lambda^X_i,
		\cr \int_Y \norm{\nabla^Y\phi^Y_i(y)}^2 d\text{V}_Y &=& \lambda^Y_i.
		\end{eqnarray}
		Plugging  Eq.\ref{eq:qc_green2} into Eq. \ref{eq:qc_int}, and using Lemma \ref{lem:jac}
		\begin{eqnarray*}
		\text{Vol}(X)\cfrac{e^{-(\lambda^X_i t)}}{\overset{}{\lambda^X_i}} \lambda^X_i &=& \text{Vol}(Y)\cfrac{e^{-(\lambda^Y_i t)}}{\overset{}{\lambda^Y_i}} \mathcal{J}_{\psi} {\lambda^Y_i},
		\cr e^{-(\lambda^X_i t)} &=& e^{-(\lambda^Y_i t)},
		\cr \lambda^X_i &=&  \lambda^Y_i. \nonumber
		\end{eqnarray*}
\end{proof}
This shows that $X$ and $Y$ are isospectral and $\text{Vol}(X) = \text{Vol}(Y)$.

\begin{proposition}
\label{prop:isom}
	For $X,Y \in \mathcal{M}$, if $d^{\text{QC}}(X, Y)=0$, then $X$, $Y$ are isometric.
\end{proposition}
\begin{proof}
	 The eigenfunctions $\phi^Y_i$ satisfy
	\begin{eqnarray*}
		\Delta^Y \phi^Y_i(y) &=& -\lambda^Y_i \phi^Y_i(y).
	\end{eqnarray*}
	From Lemma \ref{lem:isospec} and Eq. \ref{eq:qc_constraint0}
	\begin{eqnarray*}
		\Delta^Y \phi^X_i(\psi(y)) &=& -\lambda^X_i \phi^X_i(\psi(y)). \nonumber
	\end{eqnarray*}
	Hence, the two Riemannian manifolds $X$,$Y$ share common eigenfunctions $\phi^X_i$, with respective eigenvalues $\lambda^X_i$.
\end{proof}

\begin{proposition}
\label{prop:tri}
	For every three Riemannian manifolds $X, Y, Z \in \mathcal{M}$,
	\begin{eqnarray*}
		d^{QC}_t(X,Z) &\le& d^{QC}_t(X,Y) + d^{QC}_t(Y,Z). \nonumber
	\end{eqnarray*}\end{proposition}
\begin{proof}
We denote $d_{X,Y}(x,y) \equiv d(J^{\Phi^X}_t(x),J^{\Phi^Y}_t(y))$.

By definition of $d^{\text{QC}}_t(X, Y)$, 
if $d^{QC}_t(X, Y) \le\epsilon$, then for any $\Phi^X$, one can find $\Phi^Y$, $\varphi$, $\psi$ such that $\text{dis}_t^{\text{QC}}(\varphi)\le\epsilon$, $\text{dis}_t^{\text{QC}}(\psi)\le\epsilon$, \ie
		\begin{eqnarray*}
		d_{X, Y}(x,\varphi(x)) &\le& \epsilon, \,\: \forall \, x,
		\cr d_{X, Y}(\psi(y), y) &\le& \epsilon, \,\: \forall \, y. \nonumber
	\end{eqnarray*}
Let  $d^{\text{QC}}_t(X, Y)=\epsilon_1$ and  $d^{\text{QC}}_t(Y, Z)=\epsilon_2$. Hence, for any $\Phi^{X}$,  there exist two eigenbases $(\Phi^{Y}$, $\Phi^{Z})$, and two pairs of corresponding mappings $(\varphi_1 : X \mapsto Y,\psi_1 : Y \mapsto X)$ and $(\varphi_ 2: Y \mapsto Z,\psi_ 2: Z \mapsto Y)$  satisfying $\text{dis}_t^{\text{QC}}(\varphi_1)\le\epsilon_1$, $\text{dis}_t^{\text{QC}}(\psi_1)\le\epsilon_1$, and $\text{dis}_t^{\text{QC}}(\varphi_2)\le\epsilon_2$, $\text{dis}_t^{\text{QC}}(\psi_2)\le\epsilon_2$.
	Denote by $\varphi = \varphi_2 \circ \varphi_1 : X \mapsto Z$, $\psi = \psi_1 \circ \psi_2 : Z \mapsto X$.
	Invoking the triangle inequality for Euclidean spaces, one has
\begin{eqnarray*}
	d_{X, Z}(x,\varphi(x))
	&\le& \eqs d_{X, Y}(x,\varphi_1(x)) + d_{Y, Z}(\varphi_1(x), \varphi(x)) 
	\cr &\le& \eqs \text{dis}_t^{\text{QC}}(\varphi_1) + \text{dis}_t^{\text{QC}}(\varphi_2) \cr &\le& \epsilon_1 + \epsilon_2, \,\: \forall \, x \in X ,
	\cr  d_{X, Z}(\psi(z),z)
	&\le& \eqs d_{X, Y}(\psi(z),\psi_1(z)) + d_{Y, Z}(\psi_2(z), z) 
	\cr &\le& \eqs \text{dis}_t^{\text{QC}}(\psi_1) + \text{dis}_t^{\text{QC}}(\psi_2)  \cr &\le & \epsilon_1 + \epsilon_2, \,\: \forall \, z \in Z. \nonumber
\end{eqnarray*}
	This means that for any $\Phi^{X}$, we can find $\Phi^{Z},\varphi,\psi$, such that $d_{X, Z}(x,\varphi(x))$ and $d_{X, Z}(\psi(z),z)$ are bounded by $\epsilon_1 +\epsilon_2$. Consequently,
	\begin{eqnarray*}
	\inf_{\{\Phi^{Z}\}} d_\mathcal{H}^{\text{EMB}}(I^{\Phi^{X}}_t,I^{\Phi^{Z}}_t) &\le& \epsilon_1+\epsilon_2, \,\: \forall \, \Phi^{X}.
	\end{eqnarray*}
Clearly,
	\begin{eqnarray*}
		d_{J_t}^{\text{QC}}(X, Z) &=& \sup_{\{\Phi^{X}\}}\inf_{\{\Phi^{Z}\}} d_\mathcal{H}^{\text{EMB}}(I^{\Phi^{X}}_t,I^{\Phi^{Z}}_t) \eqs \le \eqs \epsilon_1 + \epsilon_2. \nonumber
	\end{eqnarray*}
In the same way $d_{J_t}^{\text{QC}}(Z, X) \le \epsilon_1 + \epsilon_2$, implying $d_t^{\text{QC}}(X, Z) \le \epsilon_1 + \epsilon_2$. \nonumber
\end{proof}

We conclude that the distance $d_t^{\text{QC}}(X, Y):\mathcal{M} \times \mathcal{M} \mapsto \mathbb{R}_{\ge 0}$ satisfies the properties of Theorem \ref{thm:pmet}.
\begin{enumerate}
	\item[P1:] Symmetry - by definition of the spectral quasi-conformal distance.
	\item[P2:] Triangle-inequality - by Proposition \ref{prop:tri}.
	\item[P3:] Identity - 
\begin{enumerate}
		\item - by the intrinsic equivalence of isometric manifolds. 
		\item - by Proposition \ref{prop:isom}.
\end{enumerate}
\end{enumerate}
$\blacksquare$

\ifdefined \TOG
\bibliographystyle{acmtog}
\else
\bibliographystyle{alpha}
\fi
\bibliography{Shape-Similarity-Correspondence}

\ifdefined \TOG
\received{RRR}{AAA}
\fi

\end{document}